\documentclass[twoside]{article}

\usepackage{aistats2023}
\usepackage[round]{natbib}

\usepackage[colorlinks,citecolor=blue]{hyperref}

\usepackage{amsmath,amssymb,amsthm}
\usepackage{dsfont}
\usepackage{graphicx}
\usepackage{color}
\usepackage{forest}
\usepackage{url}
\usepackage{booktabs}
\usepackage{enumitem}
\usepackage{pifont}
\usepackage{adjustbox}
\usepackage{svg}
\usepackage{longtable}
\usepackage{multirow}

\newtheorem{proposition}{Proposition}

\newtheorem{remark}{Remark}

\theoremstyle{definition}\newtheorem{definition}{Definition}
\newtheorem{exmp}{Example}

\newcommand{\obj}{\mathrm{obj}}
\newcommand{\var}{\operatorname*{Var}}

\newcommand{\bx}{\mathbf{x}}
\newcommand{\1}{\mathds{1}}
\newcommand{\const}{\mathrm{const}}
\newcommand{\tleft}{\mathrm{left}}
\newcommand{\tright}{\mathrm{right}}
\newcommand{\tparent}{\mathrm{parent}}

\def\X#1{%
        \raisebox{.9pt}{\textcircled{\raisebox{-.9pt}{#1}}}%
}

\begin{document}

%

%

\twocolumn[

\aistatstitle{Individualized and Global Feature Attributions for Gradient Boosted Trees in the Presence of $\ell_2$ Regularization}

\aistatsauthor{ Qingyao Sun }

\aistatsaddress{ University of Chicago } ]

\begin{abstract}
While $\ell_2$ regularization is widely used in training gradient boosted trees, popular individualized feature attribution methods for trees such as Saabas and TreeSHAP overlook the training procedure.
We propose Prediction Decomposition Attribution (PreDecomp), a novel individualized feature attribution for gradient boosted trees when they are trained with $\ell_2$ regularization.
Theoretical analysis shows that the inner product between PreDecomp and labels on in-sample data is essentially the total gain of a tree, and that it can faithfully recover additive models in the population case when features are independent.
Inspired by the connection between PreDecomp and total gain, we also propose TreeInner, a family of debiased global feature attributions defined in terms of the inner product between any individualized feature attribution and labels on out-sample data for each tree.
Numerical experiments on a simulated dataset and a genomic ChIP dataset show that TreeInner has state-of-the-art feature selection performance.
Code reproducing experiments is available at \url{https://github.com/nalzok/TreeInner}.
\end{abstract}

\section{INTRODUCTION}
The interpretability of machine learning algorithms has been receiving research interest due to its relevance to real-life problems in domains including genomics~\citep{Basu1943,Kumbier2018RefiningIS}, healthcare~\citep{Rudin2018OptimizedSS,10.1145/2783258.2788613}, safety~\citep{DBLP:phd/us/Kim19}, data privacy~\citep{https://doi.org/10.48550/arxiv.1806.03253}, and criminal justice~\citep{Rudin2020Ageof}.
Tree ensembles such as random forests~\citep{Breiman2001} and gradient boosted trees~\citep{10.2307/2699986} are of particular interest because they tend to have high predictive accuracy for tabular datasets and the tree structure is more interpretable compared to neural networks~\citep{lundberg2020local}.
In addition, tree ensembles are widely available in many open-source ML packages~\citep{Chen_2016,NIPS2017_6449f44a,NEURIPS2018_14491b75}.

In interpretable machine learning~\citep{46160}, individualized feature attribution (IFA) and global feature attribution (GFA) are two related topics~\citep{Lundberg2018}.
IFA attributes the prediction by a model for a particular input to each feature, and it helps to interpret the decision-making process for a given input.
Examples of IFA methods include the Saabas method~\citep{Saabas2014} which is designed for random forests, and SHAP~\citep{Lundberg2018} which is model-agnostic.
On the other hand, GFA focuses on estimating the overall importance of each feature in model inference, and it provides a way to understand the high-level relationship between input features and model prediction.
Examples of GFA methods include MDI~\citep{Breiman2001} \& MDI-oob~\citep{li2019debiased} which are designed for random forests, and permutation importance~\citep{Breiman2001} which is model-agnostic.

A common way to construct GFA from IFA is by averaging or summing the absolute value of the IFA for each feature across all samples.
However, this approach does not take into account the relationship between IFA and the fitting target, which makes the resultant GFA suffer from feature selection bias --- some features tend to be over-weighted or under-weighted, and thus their GFA is not reflective of their actual importance in the data.
Many researchers proposed ways to deal with this bias.
\citet{doi:10.1198/106186008X344522} proposed a bias-corrected impurity based on estimating the bias with pseudo data, for which the R package \textbf{ranger} provides an efficient algorithm and implementation~\citep{10.1093/bioinformatics/bty373}.
Some researchers ascribe the bias to over-fitting and address it by incorporating validation data.
In particular, \citet{Markus2022} and \citet{zhou2019unbiased} developed variants of MDI based on different impurity functions calculated on both in-bag and out-of-bag data.
\citet{li2019debiased} proposed MDI-oob, a debiased GFA calculated on out-of-bag data motivated by a new characterization of MDI.
There is also a separate line of research that focuses on growing trees without such bias, e.g. cForest~\citep{doi:10.1198/106186006X133933} and honest trees~\citep{doi:10.1080/01621459.2017.1319839}.
However, these papers focus on random forests, but there is no theoretical or empirical evidence that they also apply to gradient-boosted trees (GBT).
Additionally, these works only apply to trees trained without regularization, whereas it is common practice to apply $\ell_2$ regularization to base learners in the GBT training process, e.g. the $\ell_2$ regularization parameter defaults to 1 in \textbf{xgboost}.

Inspired by the ideas researched for random forests, we construct a novel IFA called PreDecomp and a family of debiased GFAs for gradient-boosted trees called TreeInner.
The major contributions of this paper are listed as follows:
\begin{itemize}
  \item We propose a novel IFA for $\ell_2$-regularized gradient boosted trees called Prediction Decomposition Attribution (PreDecomp) by generalizing the Saabas method.
  \item We derive an alternative analytical expression for the total gain DFA in gradient boosted trees and propose a family of debiased GFAs called TreeInner by calculating the inner product between any IFA and the labels on an out-sample dataset.
  \item We show that TreeInner materialized by PreDecomp and TreeSHAP achieves state-of-the-art AUC scores in noisy feature identification through simulation studies inspired by real data.
\end{itemize}

\subsection{Related Works}

\paragraph{Individualized Feature Attribution (IFA)}

Local surrogate models like LIME~\citep{10.1145/2939672.2939778} require sampling from a neighborhood, which is especially tricky for tabular data.
The Saabas method~\citep{Saabas2014} is only defined for random forests and does not take regularization into account.
SHAP~\citep{NIPS2017_7062,Lundberg2018,https://doi.org/10.48550/arxiv.1903.10464} has many ways of operationalization giving very different results~\citep{pmlr-v119-sundararajan20b}.
Specifically, using a marginal or interventional background distribution could distort the data manifold, using a conditional background distribution could let credit creep between correlated features, and using an interventional conditional background distribution provides a causal explanation but requires a background dataset~\citep{pmlr-v108-janzing20a}.
When it comes to implementation, KernelSHAP has exponential computational complexity in terms of explained features, whereas TreeSHAP, as a special case of Conditional Expectations Shapley, violates multiple axioms of Shapley values~\citep{pmlr-v119-sundararajan20b}.

\paragraph{Global Feature Attribution (GFA)}

There are a number of GFA methods developed for tree ensembles.
Permutation-based methods~\citep{Breiman2001} generally have high computational costs and require multiple replications to get an accurate estimation.
They also force the model to extrapolate to unrealistic samples when features are correlated~\citep{hooker2021unrestricted}.
Split-improvement scores such as total gain (also called MDI or Gini Importance in the context of random forests) and split count~\citep{Chen_2016} bias towards variables with high cardinality~\citep{strobl2007bias}.
To address the issue, multiple variants of MDI have been proposed to be debiased~\citep{li2019debiased} or unbiased~\citep{zhou2019unbiased,loecher2020unbiased,Markus2022} where ``unbiased'' means non-informative features will receive an importance score of zero in expectation.
We argue that  finite-sample unbiasedness is an unnatural requirement as the ground-truth value of feature importance is not well-defined~\citep{doi:10.1198/tast.2009.08199}: why should an unpredictive feature receive an importance score of zero instead of negative infinity?
In fact, as we show empirically in the simulation study, unbiasedness is of debatable desirability and assigning negative scores to noisy features can potentially improve feature selection performance.
On the other hand, SHAP feature importance, i.e. the average absolute value of SHAP across the samples for each feature, is merely heuristic.
Additionally, it also suffers from a strong dependence on feature cardinality in both random forests~\citep{10.1007/978-3-031-14463-9_8} and gradient boosted trees~\citep{e24050687}.
Apart from model-based GFAs, there also exists non-parametric techniques that don't require a fitted model to estimate GFA~\citep{PARR2021100146,parr2020nonparametric}.

We recently learned that there is a concurrent work by \citet{10.1007/978-3-031-14463-9_8} which discusses a similar idea as ours.
However, they focus on SHAP and random forests, whereas we propose a whole family of GFAs parameterized by the choice of IFA for gradient boosted trees trained with $\ell_2$ regularization.

\subsection{Organization}

The remainder of the paper is organized as follows. In Section \ref{S2}, we introduce a novel IFA called Prediction Decomposition Attribution (PreDecomp) and study its theoretical properties. In Section \ref{S3}, we derive an alternative expression of total gain in terms of PreDecomp and propose TreeInner, a family of GFA based on the formula. In Section \ref{S4}, we study our method's empirical performance and demonstrate that it has state-of-the-art feature selection performance. In Section \ref{S5}, we conclude our paper.

\section{INDIVIDUALIZED FEATURE ATTRIBUTIONS}\label{S2}

For simplicity, we focus on regression trees trained with mean squared error (MSE) loss.
However, our result also applies to classification trees and Gini index since variance and Gini index are equivalent with one-hot encoding as shown in \citet{li2019debiased}.
Without loss of generality, we also assume all input features are numerical.

\subsection{Background and Notations}

Assuming the feature space is $\mathbb R^p$, an individualized feature attribution is a function on $\mathcal{M} \times \mathbb R^p \to \mathbb R^p$, where $\mathcal{M}$ is the set of all possible models in interest.
A global feature attribution is a function on $\mathcal{M} \to \mathbb R^p$.

Denote $\mathcal{D}_\text{train} = \{(\bx_i, y_i)\}_{i=1}^N$ to be a training dataset containing iid samples, where $\bx_i = (\bx_{i,1}, \ldots, \bx_{i,p}) \in \mathbb{R}^p$ is the input features and $y_i \in \mathbb{R}$ is the response.
We may also refer to a single specimen with $(X, Y)$ where $X = (X_1, \ldots, X_p)$. 
Similarly, denote $\mathcal{D}_\text{valid}$ to be a validation dataset, which can be the validation split reserved for hyperparameter tuning for early stopping.
We define noisy features as features that are independent of the outcome.
Non-noisy features are also called relevant features.

Gradient boosted trees \citep{10.2307/2699986} is an ensemble of decision trees whose prediction is the sum prediction across all trees.
Each $t$ represents a hyper-rectangle $R_t$ in the feature space $\mathbb{R}^p$, and all samples in $R_t$ share a common prediction $w_t$.
For node $t$, we define $v(t)$ to be the feature it split on,  $t^\tleft$ to be its left child and $t^\tright$ to be its right child.
Denote $M$ to be the total number of trees or boosting rounds.
For any $m \in [M] = \{1, 2, \ldots, M\}$, denote the prediction of tree $m$ to be $f_{m}(X)$, and the prediction of the ensemble consisting of the first $m$ trees to be $f_{[m]}(X) = \sum_{q \le m} f_{q}(X)$.

Gradient boosted trees are trained sequentially and greedily: tree by tree, and split by split.
For each split $t$ in tree $m$, the objective function consists of a training loss term and a regularization term.
Particularly, denote $L(y, \hat{y}) = \frac{1}{2}(y - \hat{y})^2$ to be the MSE loss, $\lambda$ to be the $\ell_2$ penalty parameter, and the objective function is $\obj_t(w) = \sum_{\bx_i \in R_t} L(y_i, f_{[m]}(\bx_i)) + \frac{1}{2}w^2$\footnote{Unless otherwise specified, $\sum_{\bx_i \in R_t} \cdot$ means $\sum_{\bx_i \in R_t \cap \mathcal{D}_\text{train}} \cdot$}, where $f_{[m]}(\bx_i) = f_{[m-1]}(\bx_i) + f_{m}(\bx_i)$.

With second-order Taylor expansion, we can construct the gradient at sample $X$ when growing tree $m$
\begin{equation}
G_{i,m} = \frac{\partial L\left(y_i, \hat{y}_i\right)}{\partial \hat{y}_i}\bigg|_{\hat{y}_i = f_{[m-1]}(\bx_i)} = f_{[m-1]}(\bx_i) - y_i.
\end{equation}
and the hessian
\begin{equation}
\begin{aligned}
H_{i,m} = \frac{\partial^2L\left(y_i, \hat{y}_i\right)}{\partial \hat{y}_i^2}\bigg|_{\hat{y}_i = f_{[m-1]}(\bx_i)} = 1.
\end{aligned}
\end{equation}

Define $p_m(t)$ to be the prediction at a leaf node $t$ in tree $m$. By considering the second-order optimality condition \citep{Chen_2016}, we know that for a leaf node $t$
\begin{equation}
\begin{aligned}\label{eq:pmt-leaf}
p_m(t) =& \arg\min_w \obj_t(w) \\
=& - \frac{\sum_{\bx_i\in R_t}G_{i,m}}{\sum_{\bx_i\in R_t}H_{i,m} + \lambda} \\
=& \frac{\sum_{\bx_i\in R_t} \left(y_i - f_{[m-1]}(\bx_i)\right)}{|R_t| + \lambda}.
\end{aligned}
\end{equation}

In practice, people often consider an additional learning rate parameter $\alpha$ for shrinkage effects, resulting in
\begin{equation}
p_m(t) = \alpha \frac{\sum_{\bx_i\in R_t} \left(y_i - f_{[m-1]}(\bx_i)\right)}{|R_t| + \lambda}.
\end{equation}

Throughout this paper, we assume $\lambda$ and $\alpha$ are constant for all trees.

\subsection{Defining \texorpdfstring{$p_m(t)$}{p\_m(t)} for Inner Nodes}

The Saabas method was initially proposed in the context of random forests.
In particular, the original blog post by \citet{Saabas2014} breaks down the prediction value of a tree model by tracking value changes along the prediction path.
We will explain the details later in Definition~\ref{def:predecomp}.
A key step in the Saabas method is to assign a prediction value $p_m(t)$ to all nodes $t$ in tree $m$, and use the change of prediction values to attribute features used by each split.
Normally for tree models, only leaves are used for prediction and thus have a well-defined $p_m(t)$.
In this subsection, we discuss the definition of $p_m(t)$ when $t$ is an inner node.

The vanilla Saabas method calculates the prediction of each inner node by computing the average response of all training samples falling within it.
\begin{equation}
\mathring{p}_m(t) := \frac{1}{|R_t|} \sum_{\bx_i\in R_t}y_i.
\end{equation}

For random forests, this way of assigning values to an inner node is equivalent to treating it as a leaf node by ignoring all of its children nodes.
Coincidentally, doing so also aligns with the idea that leaf node predictions are to minimize the mean squared error without any regularization term.
That is to say, we have
\begin{equation}\label{eq:alternative-p}
\arg\min_w \obj(w) \stackrel{\text{\X1}}{=} \mathring{p}_m(t) \stackrel{\text{\X2}}{=} \frac{1}{|R_t|} \sum_{\bx_i\in R_t}f_m(\bx_i).
\end{equation}

Unfortunately, when leaf values are trained with regularization terms such as $\ell_2$, the average fitted value of all training samples within a node will be different from treating the inner node as a leaf node.
In other words, \X1 and \X2 in the above equation cannot hold simultaneously.
We choose to maintain \X1 in Eq.~\eqref{eq:alternative-p} and let
\begin{equation}\label{eq:pmt}
\begin{aligned}
\hat{p}_m(t) =& \arg\min_w \obj(w) \\
=& -\alpha \frac{\sum_{\bx_i\in R_t}G_{i,m}}{\sum_{\bx_i\in R_t}H_{i,m} + \lambda} \\
=& \alpha \frac{\sum_{\bx_i\in R_t} \left(y_i - f_{[m-1]}(\bx_i)\right)}{|R_t| + \lambda}.
\end{aligned}
\end{equation}

Hereafter, we denote $p_m(t) := \hat{p}_m(t)$.

\paragraph{Alternative Definition}
For the sake of completeness, we note that some software packages implement the Saabas method with an alternative definition $\tilde{p}_m(t)$ for inner nodes as the expected output of the model~\citep{lundberg2020local}.
Specifically, $\tilde{p}_m(t)$ calculates the expected prediction value for all training samples falling within an inner node by enforcing \X2 in Eq.~\eqref{eq:alternative-p}, i.e.
\begin{equation}
\tilde{p}_m(t) = \frac{1}{|R_t|} \sum_{\bx_i\in R_t}f_m(\bx_i).
\end{equation}

While $\tilde{p}_m(t)$ and $\hat{p}_m(t)$ are asymptotically equivalent, there is some difference in the finite-sample case.
We endorse $\hat{p}_m(t)$ for its nice analytical property to be demonstrated later, and because it provides a robust tree-level bias regardless of the tree structure.
See Section~\ref{S6} in the supplementary material for further discussion and an example.

\subsection{Generalizing Saabas Method}

In this subsection, we propose a generalization of the Saabas method to gradient boosted trees trained with $\ell^2$ regularization in terms of $p_m(t)$.

Now we can define $f_{m,k}(X)$, the contribution of feature $k$ in the prediction of tree $m$.

\begin{definition}[Prediction Decomposition Attribution (PreDecomp)]\label{def:predecomp}
For tree $m$ and feature $k$,
\begin{equation}
\begin{aligned}
f_{m,k}(X) = \sum_{t\in I_m:v(t)=k} & \left[p_m(t^\tleft)\1(X\in R_{t^{\tleft}})\right. \\
&+ p_m(t^\tright)\1(X\in R_{t^{\tright}}) \\
&- \left.p_m(t)\1(X\in R_{t})\right].
\end{aligned}
\end{equation}

Additionally, for the first $m$ trees,
\begin{equation}
f_{[m],k}(X) = \sum_{q \le m} f_{q,k}(X)
\end{equation}
\end{definition}

While PreDecomp looks superficially similar to the vanilla Sabbas method, we note that the key difference lies in our choice of $p_m(t)$.
The time complexity of evaluating $f_{m,k}(X)$ is the same as doing inference with the tree $m$.

Next, we point out that the prediction of each tree is just a sum of $f_{m,k}(X)$ and a constant term which can be understood as a tree-level bias.

\begin{proposition}[Local Accuracy]\label{prop:local-accuracy}
For any $m \in [M]$, we have the following equation, where $p_m(\text{root})$ is the prediction value for the root node.
\begin{equation}
\begin{aligned}
f_m(X) =& \sum_{k=1}^p f_{m,k}(X) + p_m(\text{root}) \\
=& \sum_{k=1}^p f_{m,k}(X) - \alpha\frac{\sum_{i=1}^n G_{i,m}}{\sum_{i=1}^n H_{i,m} + \lambda},
\end{aligned}
\end{equation}
\end{proposition}

\begin{proof}
See Section~\ref{S8.1} in the supplementary material.
\end{proof}

In the following theorem, we point out that our characterization can faithfully recover additive models when features are independent.


\begin{proposition}[Additive Model Consistency with Infinite Samples]\label{prop:additive}
Assume $\mathbb E(Y) = 0$ for identifiability, and suppose that
\begin{itemize}
    \item There are infinite training samples (population case);
    \item $X_1, \ldots, X_p$ are independent;
    \item The regression function $f^*(X)=\mathbb E(Y|X)$ is additive, i.e., there exists uni-variate functions $h_1,\ldots,h_p$ such that $\mathbb E_X (h_k(X_k)) = 0$ and for any $X \in \mathbb R^p$
    \begin{equation}
    f^*(X) = \sum_{k=1}^p h_k(X_k).
    \end{equation}
\end{itemize}
then for any $1 \le k \le p$, the IFA $f_{[M],k}$ can consistently recover $h_k$ as the number of trees $M \to \infty$:
\begin{equation}\label{Eq:prop2}
    \lim_{M\to\infty} \mathbb E (f_{[M], k}(X) - h_k(X_j))^2 = 0.
\end{equation}
\end{proposition}

\begin{proof}
See Section~\ref{S8.2} in the supplementary material.
Note that without the assumptions $\mathbb E(Y) = 0$ and $\mathbb E_X (h_k(X)) = 0$, the global bias can be absorbed into an arbitrary $h_k(X)$, which leads to unidentifiablity.
\end{proof}

Indeed, each $X_k$ can also be interpreted as a group of features.
There can be feature dependency and interaction within each group, as long as the groups are mutually independent and have no interaction with each other.

\section{GLOBAL FEATURE ATTRIBUTIONS}\label{S3}

Based on our understanding of PreDecomp, we can construct a family of debiased GFA.
We begin with analyzing the relationship between PreDecomp and total gain, which serves as a motivation for our construction.

\subsection{Alternative Formula for Total Gain}

In this subsection, we provide an alternative analytical expression for total gain.

\begin{definition}[Gain]
Gain, or the amount of decreased loss due to the split at node $t$ in tree $m$, is defined as the sum of all the gains due to splits using feature $k$.
\begin{equation}
\Delta_m(t) = \begin{aligned}[t]
& \frac{(\sum_{\bx_i\in R_{t^\tleft}} G_{i,m})^2}{\sum_{\bx_i\in R_{t^\tleft}} H_{i,m}+\lambda} + \frac{(\sum_{\bx_i\in R_{t^\tright}} G_{i,m})^2}{\sum_{\bx_i\in R_{t^\tright}} H_{i,m}+\lambda} \\
-& \frac{(\sum_{\bx_i\in R_{t}} G_{i,m})^2}{\sum_{\bx_i\in R_{t}} H_{i,m}+\lambda}.
\end{aligned}
\end{equation}
\end{definition}

Note again that calculating gain implicitly requires a background dataset $\mathcal{D}$, and that by summing over $\bx_i\in R_{t}$ we really mean $\bx_i\in R_{t} \cap \mathcal{D}$.
In the original definition of gain, we have $\mathcal{D} := \mathcal{D}_{\text{train}}$.

Next, we define total gain for individual trees and forests.

\begin{definition}[Total Gain]
The total gain in tree $m$ with respect to feature $k$ is defined as
\begin{equation}
\mathrm{TotalGain}_{m,k} = \sum_{t\in I_m:v(t)=k}\Delta_m(t).
\end{equation}

The total gain across the whole forest with respect to feature $k$ is defined as
\begin{equation}
\begin{aligned}
\mathrm{TotalGain}_k &= \sum_{m=1}^M \mathrm{TotalGain}_{m,k} \\
&= \sum_{m=1}^M\sum_{t\in I_m:v(t)=k}\Delta_m(t).
\end{aligned}
\end{equation}
\end{definition}

Then we have the following proposition:
\begin{proposition}\label{prop:gain}
The total gain feature attribution for tree $m$ and feature $k$ can be written as:
\begin{equation}\label{eq:gain}
\begin{aligned}
\mathrm{TotalGain}_{m,k} =& -\alpha^{-1} \sum_{i=1}^N  f_{m, k}(\bx_i) G_{i,m} \\
=& \alpha^{-1} \sum_{i=1}^N f_{m, k}(\bx_i) \left(y_i - f_{[m-1]}(\bx_i)\right).
\end{aligned}
\end{equation}
\end{proposition}

\begin{proof}
See Section~\ref{S8.3} in the supplementary material.
\end{proof}

As a corollary, we immediately have
\begin{equation}
\begin{aligned}
\mathrm{TotalGain}_k =& \sum_{m=1}^M \mathrm{TotalGain}_{m,k} \\
=& -\alpha^{-1} \sum_{i=1}^N \sum_{m=1}^M  f_{m, k}(\bx_i) G_{i,m} \\
=& \alpha^{-1} \sum_{i=1}^N \sum_{m=1}^M  f_{m, k}(\bx_i) \left(y_i - f_{[m-1]}(\bx_i)\right).
\end{aligned}
\end{equation}

It was pointed out by \citet{li2019debiased} that in random forests MDI can be written as the in-sample covariance, in addition to the inner product, between the response $y_i$ and an IFA similar to $f_{m,k}(\bx_i)$.
To compare our results, we make the following remark

\begin{remark}\label{remark:sum-to-zero}
When $\lambda = 0$, for any $k$,
\begin{equation}
\sum_{\bx \in R_\text{root}} f_m(\bx) = 0.
\end{equation}
\end{remark}

\begin{proof}
See Section~\ref{S8.4} in the supplementary material.
\end{proof}

When $\lambda = 0$, Proposition~\ref{prop:gain} can also be interpreted as the in-sample covariance between the gradient and $f_{m,k}(\bx_i)$ by Remark~\ref{remark:sum-to-zero}.
For $\lambda > 0$, we note that it is still approximately a covariance since the regularization term becomes negligible with a sufficiently large sample size.

\subsection{Bridging IFA and GFA}

Inspired by Proposition~\ref{prop:gain} which shows the total gain of each tree is just the inner product of the label and the IFA for the in-sample data, we generalize the idea to any IFA and any dataset (possibly out-sample).

\begin{definition}[TreeInner]
For any IFA $r_{m,k}(\bx)$ defined for tree $m$ and feature $k$, we can define a tree-level GFA as follows
\begin{equation}
\begin{aligned}
\operatorname{GFA}_r(m,k) =& -\alpha^{-1}\sum_{i=1}^N r_{m,k}(\bx_i) G_i \\
=& \alpha^{-1}\sum_{i=1}^N r_{m,k}(\bx_i) (y_i - f_{[m-1]}(\bx_i)).
\end{aligned}
\end{equation}

Based on $\operatorname{GFA}_r(m,k)$, we can define a GFA for the entire forest
\begin{equation}
\begin{aligned}
\operatorname{GFA}_r(k) =& \sum_{m=1}^M \operatorname{GFA}_r(m,k) \\
=& \alpha^{-1}\sum_{i=1}^N \sum_{m=1}^M \left(r_{m,k}(\bx_i) (y_i - f_{[m-1]}(\bx_i))\right).
\end{aligned}
\end{equation}
\end{definition}

Note that we can evaluate the importance of a feature on the out-sample data, which mitigates a common feature selection bias existing for previous GFA methods~\citep{li2019debiased,loecher2020unbiased}.
In other words, we sum over $\bx_i \in R_t \cap \mathcal{D}_\text{valid}$ instead of $\bx_i \in R_t \cap \mathcal{D}_\text{train}$.

\paragraph{Alternative GFA Family}
We have considered another way of inducing GFA.
It works by first aggregating the tree-level IFA into a forest-level IFA, and then calculating its inner product with the label.
In particular,
\begin{definition}[ForestInner]
\begin{equation}\label{eq:forest-inner}
\begin{aligned}
& \widetilde{\operatorname{GFA}}_r(k) \\
=& \alpha^{-1} \sum_{i=1}^N \left(\sum_{m=1}^M r_{m,k}(\bx_i)\right) \left(\sum_{m=1}^M y_i - f_{[m-1]}(\bx_i) \right) \\
=& \alpha^{-1} \sum_{i=1}^N \left(\sum_{m=1}^M r_{m,k}(\bx_i)\right) y_i.
\end{aligned}
\end{equation}
\end{definition}

When using PreDecomp as the IFA, $\widetilde{\operatorname{GFA}}_f(k)$ can be shown to be asymptotically unbiased for noisy features in the population case.
However, it empirically has worse performance, so we move the discussion about it to Section~\ref{S7} and the relevant experimental result to Section~\ref{S9} in the supplementary material.

\section{SIMULATION STUDY}\label{S4}
In this section, we conduct several numerical experiments to evaluate the empirical performance of our theoretical results.

\subsection{Datasets}

We adopt the two synthetic datasets introduced by \citet{li2019debiased}, each containing a regression task and a binary classification task.
The authors replicated each experiment 40 times and report the average result, but due to constraints on the computational budget, we only replicate the experiments 20 times.
The data-generating process is described as follows.
Note that this simulation setting has a low signal-to-noise ratio by design.

\paragraph{Simulated Dataset}
The first simulated dataset consists of 1000 (training) + 1000 (validation) samples and 50 features.
All features are discrete, and the $j$-th feature takes values in $0, 1, \ldots, j$ with equal probability.
A set $S$ of 5 features is randomly sampled from the first 10 features as relevant features, and the other 45 are treated as noise.
This setup represents a challenging case since split-improvement scores bias toward features with high cardinality, whereas the relevant features here all have low cardinality.
All samples are i.i.d. and all features are independent.
The labels are generated using the following rules:

\begin{description}[style=nextline]
    \item[Classification]
    $P(Y = 1 | X) = \textrm{Logistic}\left(\frac{2}{5}\sum_{j \in S} X_j/j - 1\right).$
    \item[Regression]
    $\begin{aligned}
        Y &= \frac{1}{5}\sum_{j\in S} X_j/j + \epsilon\text{, where} \\
        \epsilon &\sim \mathcal N\left(0, 100\cdot \var\left(\frac{1}{5}\sum_{j\in S}X_j/j\right)\right).
    \end{aligned}$
\end{description}

\paragraph{Genomic ChIP Dataset}
The second dataset is derived from a ChIP-chip and ChIP-seq dataset measuring the enhancer status of 80 biomolecules at 3912 (training) + 3897 (validation) regions of the Drosophila genome.
The heterogeneity and dependencies among features make it a more realistic and more challenging dataset for feature selection problems.
To evaluate feature selection in the ChIP data, all features are min-max scaled between 0 and 1.
Next, a set $S$ of 5 features is randomly sampled as relevant features, and the other 75 are treated as noise.
Then, all noisy features are permuted to break their dependencies with relevant features and each other, so that a noisy feature will not be considered relevant due to its correlation with a relevant feature.
The labels are generated using the following rules:

\begin{description}[style=nextline]
    \item[Classification]
    $P(Y = 1 | X) = \textrm{Logistic}\left(\frac{2}{5}\sum_{j \in S} X_j - 1\right)$.
    \item[Regression]
    $\begin{aligned}
        Y &= \frac{1}{5}\sum_{j\in S} X_j + \epsilon\text{, where} \\
        \epsilon &\sim \mathcal N\left(0, 100\cdot \var\left(\frac{1}{5}\sum_{j\in S}X_j\right)\right).
    \end{aligned}$
\end{description}

\subsection{Setup}

Throughout the experiments, we use a fork\footnote{Our fork will be open-sourced after the review period. For the time being, please use the pre-built wheel file in the supplementary material to reproduce our experiment.} of \textbf{xgboost} release 1.6.2 to calculate PreDecomp.
The hyperparameters below are described following the naming convention of the \textbf{xgboost} API.

The experiments were executed in parallel on five TPU v3-8 virtual machines accessed through Google's TPU Research Cloud program.
Each machine is equipped with 96 virtual cores and 340GiB main memory.

To explore the performance of each GFA family under different over-fitting severity, we train the model with different hyperparameters as listed in Table~\ref{table:hyperparameters}.
Since a grid search would be prohibitively expensive, we alter the hyperparameters one at a time and study their marginal effect.
In particular, we assign a standard value to every parameter when it is not the variable in interest.
For example, when considering the impact of $\texttt{eta}$, we fix $\texttt{max\_depth}$ to 4, $\texttt{min\_child\_weight}$ to 1, $\texttt{num\_boost\_round}$ to 400, $\texttt{reg\_lambda}$ to 1, and vary $\texttt{eta}$ from $10^{-5}$ all the way to $10^{0}$.
The standard values are chosen such that the model roughly has the best predictive performance for all tasks on the validation set.

\begin{table}[ht]
\begin{center}
\caption{Hyperparameters}\label{table:hyperparameters}
    \begin{tabular}{r|ll}
    \toprule
    \textbf{Name} & \textbf{Standard} & \textbf{Alternatives} \\
    \midrule
    \texttt{eta} & $10^{-2}$ & $\begin{aligned}[t](&10^{-5}, 10^{-4}, 10^{-3},\\ &10^{-2}, 10^{-1}, 10^{0})\end{aligned}$ \\
    \texttt{max\_depth} & 4 & $(2, 4, 6, 8, 10)$ \\
    \texttt{min\_child\_weight} & 1 & $(0.5, 1, 2, 4, 8)$ \\
    \texttt{num\_boost\_round} & 400 & $\begin{aligned}[t](&200, 400, 600, \\&800, 1000)\end{aligned}$ \\
    \texttt{reg\_lambda} & 1 & $(0, 0.1, 1, 10, 100)$ \\
    \bottomrule
    \end{tabular}
\end{center}
\end{table}

Due to space limitations, we only show the result with varying $\texttt{max\_depth}$ in the main text.
Additional plots showing the result of sweeping over other hyperparameters can be found in the supplementary material.

\subsection{Verification of the Proposed Formula for Total Gain}
To empirically verify the correctness of Proposition~\ref{prop:gain}, we compare the result obtained with Eq.~\eqref{eq:gain} against the output of \textbf{xgboost}'s built-in method \texttt{get\_score(importance\_type="total\_gain")}.
Since we are not interested in the absolute magnitude of total gain, the result of each method are normalized such that they sum to one, i.e. have unit $\ell_1$ norm in $\mathbb R^p$.
Figure~\ref{fig:error} illustrates the maximum absolute difference between these two methods as a measure of the deviation.
The error is mostly in the order $10^{-6}$, and it is likely caused by the inexactness of floating-point arithmetic.

\begin{figure}[ht]
\vspace{.3in}
\includegraphics[width=0.45\textwidth]{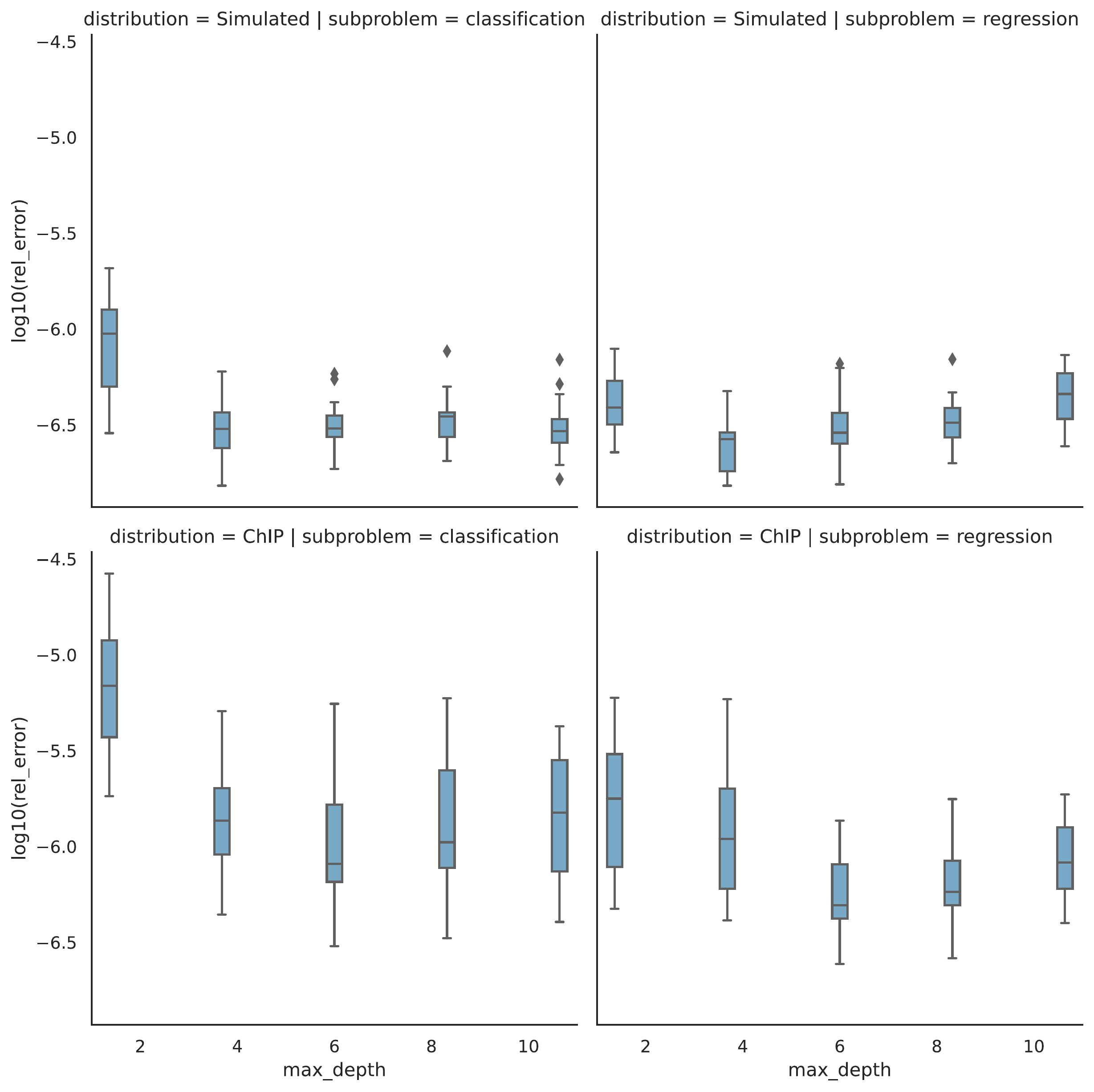}
\vspace{.3in}
\caption{Logarithm of the maximum absolute difference between the normalized total gain calculated with our methodology and the built-in method.}
\label{fig:error}
\end{figure}

\subsection{Noisy Feature Identification}
In this experiment, we evaluate the faithfulness of TreeInner to identify noisy features.
Labeling the noisy features as 0 and the relevant features as 1, we can evaluate the scores produced by a GFA in terms of its area under the receiver operating characteristic curve (AUC).
Ideally, the GFA values for relevant features would be higher than those for noisy features, so we expect to get an AUC close to 1.

We compare three GFA families:
\texttt{TreeInner} is the debiased GFA family we proposed;
\texttt{Abs} is the conventional method of calculating the average magnitude of an IFA;
\texttt{Permutation} is the permutation importance.
We evaluate the GFA on both in-sample data and out-sample data, and each family is materialized with two IFAs: PreDecomp and TreeSHAP.
Additionally, the permutation importance is included as a baseline.
The AUC scores for all methods with a model trained with standard hyperparameters can be found in Table~\ref{table:standard-auc} in Section~\ref{S9} in the supplementary material.
Note that the table includes an additional GFA family \texttt{ForestInner} as defined in Eq.~\eqref{eq:forest-inner}, which is omitted here to avoid cluttering the plot.

Figure~\ref{fig:auc} illustrates the change in AUC scores as the model becomes more and more complex with increasing \texttt{max\_depth}.
Our proposed GFA family \texttt{TreeInner} dominates other GFAs in all tasks except classification on ChIP, in which case the error bars are very wide.
Note that using out-sample data generally gives a higher AUC, which agrees with our belief that feature selection bias comes from over-fitting.
Finally, observe that when calculating \texttt{TreeInner} on out-sample data, PreDecomp generally gives the same or higher AUC compared to TreeSHAP.
\begin{figure}[ht]
\vspace{.3in}
\includegraphics[width=0.45\textwidth]{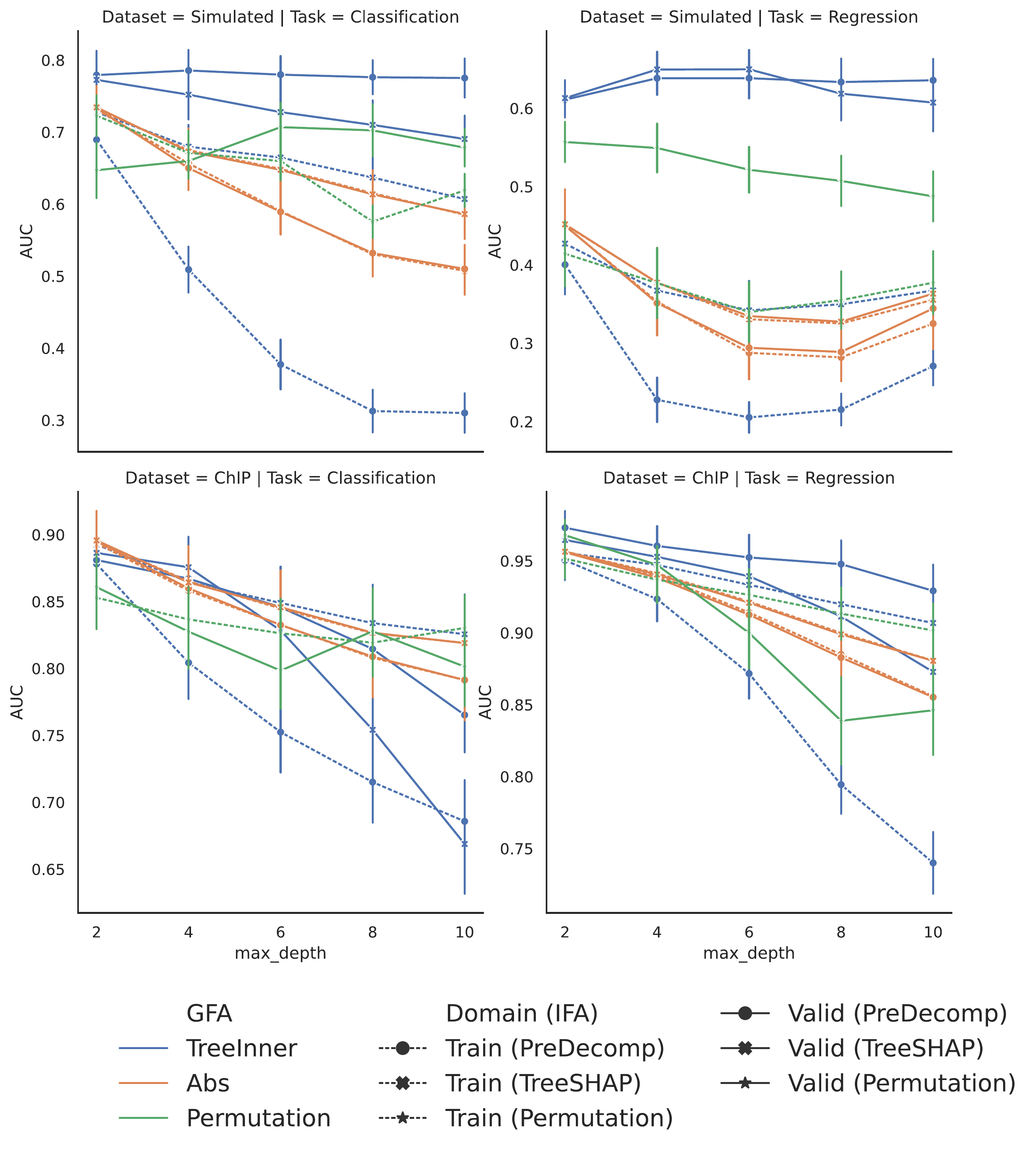}
\vspace{.3in}
\caption{
AUC score for noisy feature identification, averaged across 20 replications.
The error bars correspond to one standard error.
Higher is better.
}
\label{fig:auc}
\end{figure}

We attempt to explain the efficacy of \texttt{TreeInner} by plotting the average normalized score for noisy features in Figure~\ref{fig:score-noisy}.
The scores are normalized by scaling them to have unit $\ell_2$ norm in $\mathbb R^p$.
Note that scores calculated on in-sample data are generally positive despite being calculated for noisy features, a consequence of over-fitting.
When evaluated on out-sample data, only \texttt{TreeInner} can consistently produce negative scores, hence separating the scores for noisy and relevant features.
We hypothesize that it is the lack of finite-sample unbiasedness that makes \texttt{TreeInner} a better feature selector.
Moreover, the trend becomes clearer as the number of trees increases, suggesting that \texttt{TreeInner} could even benefit from model over-fitting.

\begin{figure}[ht]
\vspace{.3in}
\includegraphics[width=0.45\textwidth]{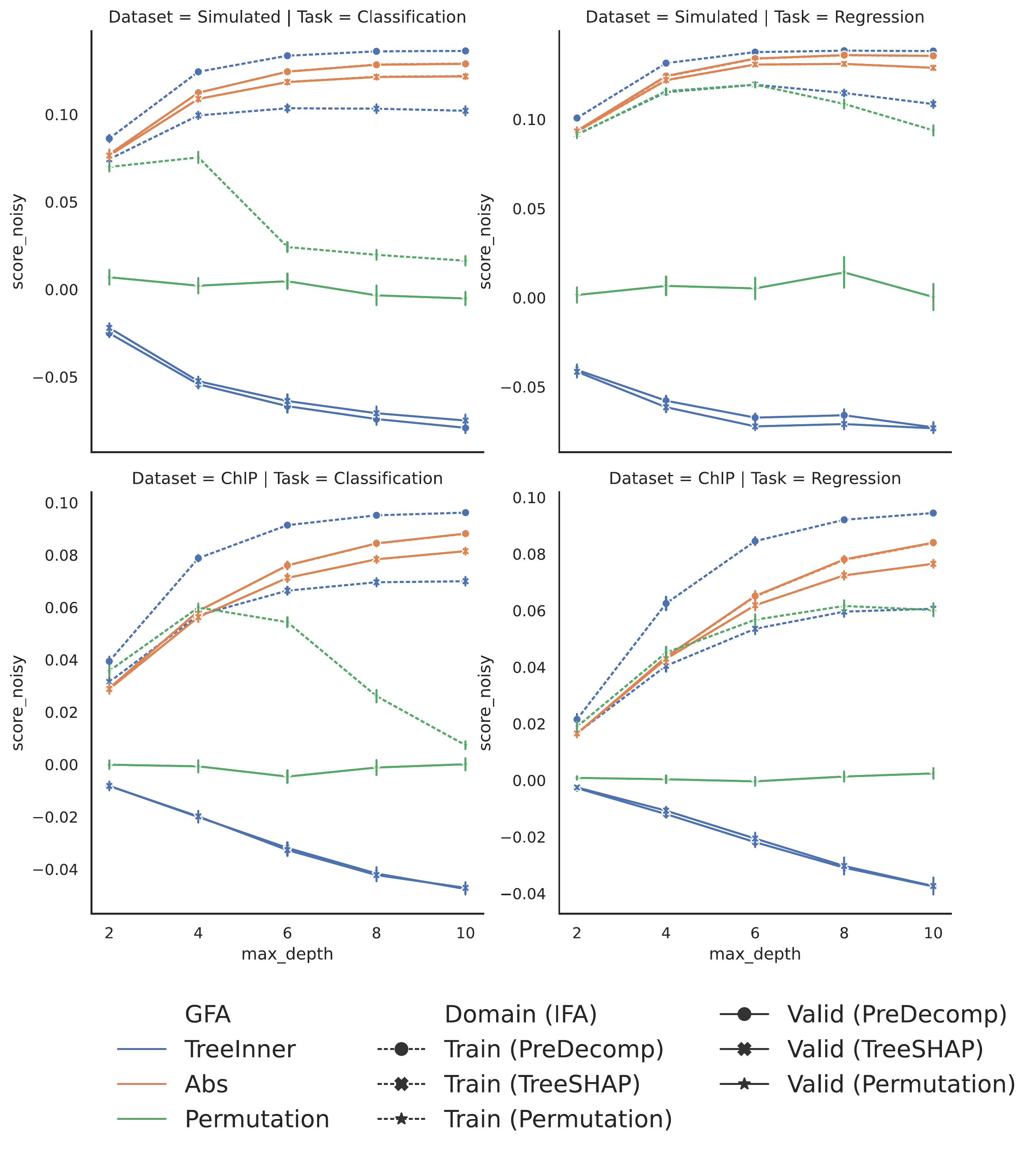}
\vspace{.3in}
\caption{
The average normalized score for the noisy features, averaged across 20 replications.
The error bars correspond to one standard error.
Lower is better.
}
\label{fig:score-noisy}
\end{figure}


\section{DISCUSSION}\label{S5}
We proposed PreDecomp, a novel IFA for gradient boosted trees in the presence of $\ell_2$ regularization, and TreeInner, a principled way to induce GFA from any IFA and background dataset.
When partnered together, PreDecomp and TreeInner show state-of-the-art feature selection performance on two real-life inspired synthetic datasets.

\subsection{Limitations}
In the presence of $\ell_1$ regularization, while one can still define PreDecomp as $\arg\min_w \obj(w)$, we are unable to express total gain in terms of PreDecomp in that case due to the soft thresholding effect.

Like interventional TreeSHAP, our PreDecomp also requires access to a background dataset.
Additionally, the time complexity of calculating total gain with Eq.~\eqref{eq:gain} is $O(N)$ where $N$ is the number of training samples, whereas using the definition gives $O(1)$ complexity.

While the experiment provides some empirical insight into how the choice of IFA could affect the performance of TreeInner, further analysis is required to understand the relationship in detail.

\subsection{Societal Impact}
Gradient boosted trees are one of the most popular algorithms for tabular data.
As the winning solution to multiple machine learning challenges in medicine, fraud detection, travel safety, and business, to name just a few, they are responsible for making many high-stake decisions~\citep{Chen_2016}.
The IFA we proposed, i.e. PreDecomp, makes these decisions more transparent and trustworthy by identifying the key features which lead to the judgment.
The family of GFA we proposed, i.e. TreeInner, can be used for feature selection, which can serve as building blocks for more predictive and robust algorithms.

Additionally, we can use gradient boosted trees as a lens to understand the importance of features in a dataset by first fitting a model and then calculating GFA with TreeInner.

\subsubsection*{Acknowledgements}
This paper is a continuation of the work I did with Dr. Karl Kumbier and Dr. Yu (Hue) Wang from the Yu Group at UC Berkeley.
I thank Prof. Bin Yu and members of the Yu Group for the inspiration and discussion.
In particular, I thank Hue for his guidance and significant contribution to this work.
Additionally, I thank Dr. Sayna Ebriahimi and Dr. Kevin Murphy from Google Brain for their feedback.
Finally, I thank Google Cloud for donating the computational resources via the TRC program.

\bibliography{refs}

\begin{thebibliography}{}

\bibitem[Aas et~al., 2019]{https://doi.org/10.48550/arxiv.1903.10464}
Aas, K., Jullum, M., and Løland, A. (2019).
\newblock Explaining individual predictions when features are dependent: More
  accurate approximations to shapley values.

\bibitem[Adler and Painsky, 2022]{e24050687}
Adler, A.~I. and Painsky, A. (2022).
\newblock Feature importance in gradient boosting trees with cross-validation
  feature selection.
\newblock {\em Entropy}, 24(5).

\bibitem[Alekh, 2018]{https://doi.org/10.48550/arxiv.1806.03253}
Alekh, S. (2018).
\newblock Eu general data protection regulation: A gentle introduction.

\bibitem[Basu et~al., 2018]{Basu1943}
Basu, S., Kumbier, K., Brown, J.~B., and Yu, B. (2018).
\newblock Iterative random forests to discover predictive and stable high-order
  interactions.
\newblock {\em Proceedings of the National Academy of Sciences},
  115(8):1943--1948.

\bibitem[Breiman, 2001]{Breiman2001}
Breiman, L. (2001).
\newblock {Random Forests}.
\newblock {\em Machine Learning}, 45:1--33.

\bibitem[Caruana et~al., 2015]{10.1145/2783258.2788613}
Caruana, R., Lou, Y., Gehrke, J., Koch, P., Sturm, M., and Elhadad, N. (2015).
\newblock Intelligible models for healthcare: Predicting pneumonia risk and
  hospital 30-day readmission.
\newblock In {\em Proceedings of the 21th ACM SIGKDD International Conference
  on Knowledge Discovery and Data Mining}, KDD '15, page 1721–1730, New York,
  NY, USA. Association for Computing Machinery.

\bibitem[Chen and Guestrin, 2016]{Chen_2016}
Chen, T. and Guestrin, C. (2016).
\newblock Xgboost.
\newblock {\em Proceedings of the 22nd ACM SIGKDD International Conference on
  Knowledge Discovery and Data Mining}.

\bibitem[Doshi-Velez and Kim, 2017]{46160}
Doshi-Velez, F. and Kim, B. (2017).
\newblock Towards a rigorous science of interpretable machine learning.
\newblock {\em arXiv}.

\bibitem[Friedman, 2001]{10.2307/2699986}
Friedman, J.~H. (2001).
\newblock Greedy function approximation: A gradient boosting machine.
\newblock {\em The Annals of Statistics}, 29(5):1189--1232.

\bibitem[Grömping, 2009]{doi:10.1198/tast.2009.08199}
Grömping, U. (2009).
\newblock Variable importance assessment in regression: Linear regression
  versus random forest.
\newblock {\em The American Statistician}, 63(4):308--319.

\bibitem[Hooker et~al., 2021]{hooker2021unrestricted}
Hooker, G., Mentch, L., and Zhou, S. (2021).
\newblock Unrestricted permutation forces extrapolation: variable importance
  requires at least one more model, or there is no free variable importance.
\newblock {\em Statistics and Computing}, 31(6):1--16.

\bibitem[Hothorn et~al., 2006]{doi:10.1198/106186006X133933}
Hothorn, T., Hornik, K., and Zeileis, A. (2006).
\newblock Unbiased recursive partitioning: A conditional inference framework.
\newblock {\em Journal of Computational and Graphical Statistics},
  15(3):651--674.

\bibitem[Janzing et~al., 2020]{pmlr-v108-janzing20a}
Janzing, D., Minorics, L., and Bloebaum, P. (2020).
\newblock Feature relevance quantification in explainable ai: A causal problem.
\newblock In Chiappa, S. and Calandra, R., editors, {\em Proceedings of the
  Twenty Third International Conference on Artificial Intelligence and
  Statistics}, volume 108 of {\em Proceedings of Machine Learning Research},
  pages 2907--2916. PMLR.

\bibitem[Ke et~al., 2017]{NIPS2017_6449f44a}
Ke, G., Meng, Q., Finley, T., Wang, T., Chen, W., Ma, W., Ye, Q., and Liu,
  T.-Y. (2017).
\newblock Lightgbm: A highly efficient gradient boosting decision tree.
\newblock In Guyon, I., Luxburg, U.~V., Bengio, S., Wallach, H., Fergus, R.,
  Vishwanathan, S., and Garnett, R., editors, {\em Advances in Neural
  Information Processing Systems}, volume~30. Curran Associates, Inc.

\bibitem[Kim, 2019]{DBLP:phd/us/Kim19}
Kim, J. (2019).
\newblock {\em Explainable and Advisable Learning for Self-driving Vehicles}.
\newblock PhD thesis, University of California, Berkeley, {USA}.

\bibitem[Kumbier et~al., 2018]{Kumbier2018RefiningIS}
Kumbier, K., Basu, S., Brown, J., Celniker, S., and Yu, B. (2018).
\newblock Refining interaction search through signed iterative random forests.
\newblock {\em bioRxiv}.

\bibitem[Li et~al., 2019]{li2019debiased}
Li, X., Wang, Y., Basu, S., Kumbier, K., and Yu, B. (2019).
\newblock A debiased mdi feature importance measure for random forests.
\newblock {\em Advances in Neural Information Processing Systems}, 32.

\bibitem[Loecher, 2020]{loecher2020unbiased}
Loecher, M. (2020).
\newblock From unbiased mdi feature importance to explainable ai for trees.
\newblock {\em arXiv}.

\bibitem[Loecher, 2022a]{10.1007/978-3-031-14463-9_8}
Loecher, M. (2022a).
\newblock Debiasing mdi feature importance and shap values in tree ensembles.
\newblock In Holzinger, A., Kieseberg, P., Tjoa, A.~M., and Weippl, E.,
  editors, {\em Machine Learning and Knowledge Extraction}, pages 114--129,
  Cham. Springer International Publishing.

\bibitem[Loecher, 2022b]{Markus2022}
Loecher, M. (2022b).
\newblock Unbiased variable importance for random forests.
\newblock {\em Communications in Statistics - Theory and Methods},
  51(5):1413--1425.

\bibitem[Lundberg et~al., 2020]{lundberg2020local}
Lundberg, S.~M., Erion, G., Chen, H., DeGrave, A., Prutkin, J.~M., Nair, B.,
  Katz, R., Himmelfarb, J., Bansal, N., and Lee, S.-I. (2020).
\newblock From local explanations to global understanding with explainable ai
  for trees.
\newblock {\em Nature machine intelligence}, 2(1):56--67.

\bibitem[Lundberg et~al., 2018]{Lundberg2018}
Lundberg, S.~M., Erion, G.~G., and Lee, S.-I. (2018).
\newblock {Consistent Individualized Feature Attribution for Tree Ensembles}.
\newblock {\em ArXiv e-prints arXiv:1802.03888}.

\bibitem[Lundberg and Lee, 2017]{NIPS2017_7062}
Lundberg, S.~M. and Lee, S.-I. (2017).
\newblock A unified approach to interpreting model predictions.
\newblock In Guyon, I., Luxburg, U.~V., Bengio, S., Wallach, H., Fergus, R.,
  Vishwanathan, S., and Garnett, R., editors, {\em Advances in Neural
  Information Processing Systems 30}, pages 4765--4774. Curran Associates, Inc.

\bibitem[Nembrini et~al., 2018]{10.1093/bioinformatics/bty373}
Nembrini, S., König, I.~R., and Wright, M.~N. (2018).
\newblock {The revival of the Gini importance?}
\newblock {\em Bioinformatics}, 34(21):3711--3718.

\bibitem[Parr and Wilson, 2021]{PARR2021100146}
Parr, T. and Wilson, J.~D. (2021).
\newblock Partial dependence through stratification.
\newblock {\em Machine Learning with Applications}, 6:100146.

\bibitem[Parr et~al., 2020]{parr2020nonparametric}
Parr, T., Wilson, J.~D., and Hamrick, J. (2020).
\newblock Nonparametric feature impact and importance.

\bibitem[Prokhorenkova et~al., 2018]{NEURIPS2018_14491b75}
Prokhorenkova, L., Gusev, G., Vorobev, A., Dorogush, A.~V., and Gulin, A.
  (2018).
\newblock Catboost: unbiased boosting with categorical features.
\newblock In Bengio, S., Wallach, H., Larochelle, H., Grauman, K.,
  Cesa-Bianchi, N., and Garnett, R., editors, {\em Advances in Neural
  Information Processing Systems}, volume~31. Curran Associates, Inc.

\bibitem[Ribeiro et~al., 2016]{10.1145/2939672.2939778}
Ribeiro, M.~T., Singh, S., and Guestrin, C. (2016).
\newblock "why should i trust you?": Explaining the predictions of any
  classifier.
\newblock In {\em Proceedings of the 22nd ACM SIGKDD International Conference
  on Knowledge Discovery and Data Mining}, KDD '16, page 1135–1144, New York,
  NY, USA. Association for Computing Machinery.

\bibitem[Rudin and Ustun, 2018]{Rudin2018OptimizedSS}
Rudin, C. and Ustun, B. (2018).
\newblock Optimized scoring systems: Toward trust in machine learning for
  healthcare and criminal justice.
\newblock {\em Interfaces}, 48:449--466.

\bibitem[Rudin et~al., 2020]{Rudin2020Ageof}
Rudin, C., Wang, C., and Coker, B. (2020).
\newblock The {Age} of {Secrecy} and {Unfairness} in {Recidivism} {Prediction}.
\newblock {\em Harvard Data Science Review}, 2(1).
\newblock https://hdsr.mitpress.mit.edu/pub/7z10o269.

\bibitem[Saabas, 2014]{Saabas2014}
Saabas, A. (2014).
\newblock {Interpreting random forests}.

\bibitem[Sandri and Zuccolotto, 2008]{doi:10.1198/106186008X344522}
Sandri, M. and Zuccolotto, P. (2008).
\newblock A bias correction algorithm for the gini variable importance measure
  in classification trees.
\newblock {\em Journal of Computational and Graphical Statistics},
  17(3):611--628.

\bibitem[Strobl et~al., 2007]{strobl2007bias}
Strobl, C., Boulesteix, A.-L., Zeileis, A., and Hothorn, T. (2007).
\newblock Bias in random forest variable importance measures: Illustrations,
  sources and a solution.
\newblock {\em BMC bioinformatics}, 8(1):1--21.

\bibitem[Sundararajan and Najmi, 2020]{pmlr-v119-sundararajan20b}
Sundararajan, M. and Najmi, A. (2020).
\newblock The many shapley values for model explanation.
\newblock In III, H.~D. and Singh, A., editors, {\em Proceedings of the 37th
  International Conference on Machine Learning}, volume 119 of {\em Proceedings
  of Machine Learning Research}, pages 9269--9278. PMLR.

\bibitem[Wager and Athey, 2018]{doi:10.1080/01621459.2017.1319839}
Wager, S. and Athey, S. (2018).
\newblock Estimation and inference of heterogeneous treatment effects using
  random forests.
\newblock {\em Journal of the American Statistical Association},
  113(523):1228--1242.

\bibitem[Zhou and Hooker, 2021]{zhou2019unbiased}
Zhou, Z. and Hooker, G. (2021).
\newblock Unbiased measurement of feature importance in tree-based methods.
\newblock {\em ACM Transactions on Knowledge Discovery from Data (TKDD)},
  15(2):1--21.

\end{thebibliography}

\onecolumn
\aistatstitle{Supplementary Materials}

\section{COMPARING \texorpdfstring{$\hat{p}_m(t)$}{\hat\{p\}\_m(t)} AND \texorpdfstring{$\tilde{p}_m(t)$}{\hat\{p\}\_m(t)}}\label{S6}
To provide some insight into the relationship between $\hat{p}_m(t)$ and $\tilde{p}_m(t)$, we note that both can be equivalently defined with a recurrence formula
\begin{equation}
\begin{aligned}
\hat{p}_m(t) &= \begin{cases}
\alpha \frac{\sum_{\bx_i\in R_t}\left(y_i - f_{[m-1]}(\bx_i)\right)}{|R_t| + \lambda},&\text{$t$ is a leaf node} \\
\begin{aligned}[t]
& \frac{|R_\tleft|+\lambda}{|R_t|+\lambda}\hat{p}_m(R_\tleft) \\
&+ \frac{|R_\tright|+\lambda}{|R_t|+\lambda}\hat{p}_m(R_\tright)
\end{aligned},&\text{$t$ is an inner node}
\end{cases}. \\
\tilde{p}_m(t) &= \begin{cases}
\alpha \frac{\sum_{\bx_i\in R_t}\left(y_i - f_{[m-1]}(\bx_i)\right)}{|R_t| + \lambda},&\text{$t$ is a leaf node}\\
\begin{aligned}[t]
& \frac{|R_\tleft|}{|R_t|}\tilde{p}_m(R_\tleft) \\
&+ \frac{|R_\tright|}{|R_t|}\tilde{p}_m(R_\tright)
\end{aligned},&\text{$t$ is an inner node}
\end{cases}.
\end{aligned}
\end{equation}

It is easy to see that $\hat{p}_m(t) = \tilde{p}_m(t)$ if and only if $\lambda = 0$.
However, for any fixed tree structure, their difference converges to zero as the sample size goes to infinity.

In \textbf{xgboost} and \textbf{shap.TreeExplainer}, a Saabas-like IFA induced by $\tilde{p}_m(t)$ is used as an approximation of TreeSHAP~\citep{Chen_2016,Lundberg2018}.
While $\tilde{p}_m(t)$ has its own merit, we argue that it is better to define IFA in terms of $\hat{p}_m(t)$.
As we showed in Proposition~\ref{prop:gain}, our $\hat{p}_m(t)$ has some nice analytical relationship with total gain.
Additionally, $\hat{p}_m(t)$ provides a robust prediction value for the root node, as illustrated by Example~\ref{exmp:root}.
Moreover, $\tilde{p}_m(t)$ has to be computed during inference time since the leaf values are unknown during training, while $\hat{p}_m(t)$ can be computed and cached in training time.

\begin{exmp}[$\hat{p}_m(t)$ offers robust tree-level bias]
\label{exmp:root}
To further illustrate why we favor $\hat{p}_m(t)$ over $\tilde{p}_m(t)$, consider the following artificial data set consisting of three specimens A, B, and C:
\begin{table}[h]
\begin{center}
\caption{Sample Dataset}\label{table:sample-data}
    \begin{tabular}{r|rrr}
    \toprule
    \textbf{Specimen} & $X_1$ & $X_2$ & $Y$ \\
    \midrule
    A & 0 & 0 & 0 \\ 
    B & 0 & 1 & 1 \\ 
    C & 1 & 0 & -1 \\
    \bottomrule
    \end{tabular}
\end{center}
\end{table}

If we train a tree with $\lambda = 1$ and maximum depth 1, the two possible tree structures are shown in Figure~\ref{fig:tree-structures}.
Since $X_1$ and $X_2$ are equally predictive in this example, the tree can split on either variable, so the training procedure would result in Structure I and Structure II with equal possibility depending on the random seed.
In practice, the instability of tree structure could stem from perturbations in $Y$ and column sub-sampling.

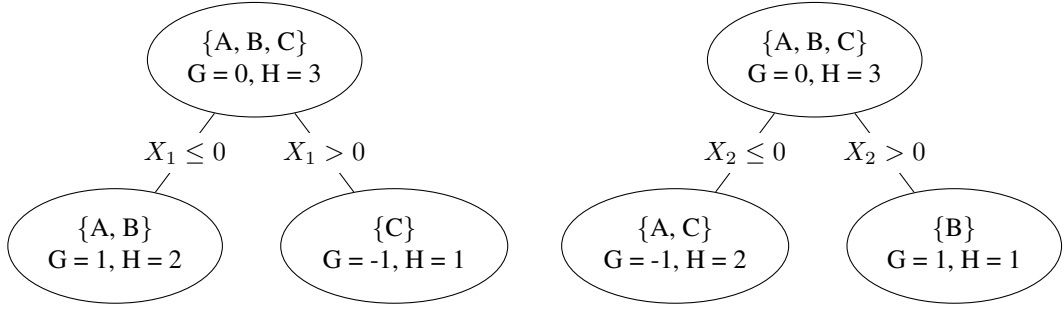
\begin{figure}
\centering
    \begin{adjustbox}{valign=t}
    \begin{forest} baseline, for tree={draw, ellipse, align=center, s sep+=15pt, l sep+=15pt}
        [{\{A, B, C\} \\ G = 0, H = 3}
            [{\{A, B\} \\ G = 1, H = 2}, edge label={node[midway,fill=white]{$X_1 \le 0$}}]
            [{\{C\} \\ G = -1, H = 1}, edge label={node[midway,fill=white]{$X_1 > 0$}}]
        ]    
    \end{forest}
    \end{adjustbox}\qquad
    \begin{adjustbox}{valign=t}
    \begin{forest} baseline, for tree={draw, ellipse, align=center, s sep+=15pt, l sep+=15pt}
        [{\{A, B, C\} \\ G = 0, H = 3}
            [{\{A, C\} \\ G = -1, H = 2}, edge label={node[midway,fill=white]{$X_2 \le 0$}}]
            [{\{B\} \\ G = 1, H = 1}, edge label={node[midway,fill=white]{$X_2 > 0$}}]
        ]    
    \end{forest}
    \end{adjustbox}
    \label{fig:tree-structures}
    \caption{Structure I (left) and Structure II (right)}
\end{figure}

We calculate $\hat{p}_m(t)$ and $\tilde{p}_m(t)$ for the nodes in both trees in Table~\ref{table:pmt}.
We focus on $\hat{p}_m(\text{root})$, since it is the constant term in Proposition~\ref{prop:local-accuracy}, which means it can be interpreted as a tree-level bias.
Observe that when using $\hat{p}_m(t)$, the prediction value of the root node is always 0 regardless of the tree structure, whereas $\tilde{p}_m(t)$ gives less stable result by alternating between $-\frac{1}{18}$ and $\frac{1}{18}$.
Indeed, this is always true for $\hat{p}_m(t)$, as we have $\hat{p}_m(\text{root}) = -\alpha\frac{\sum_i G_{i,m}}{\sum_i H_{i,m} + \lambda}$ regardless of the tree hyperparameters and structure.

\begin{table}[h]
\caption{Values of $p_m(t)$}\label{table:pmt}
\begin{center}
\begin{tabular}{lrr|rr}
\toprule
\textbf{Definition} & \multicolumn{2}{c}{$\hat{p}_m(t)$} & \multicolumn{2}{c}{$\tilde{p}_m(t)$} \\
\cmidrule(lr){2-3}\cmidrule(lr){4-5}
\textbf{Structure} & Structure I & Structure II & Structure I & Structure II \\
\midrule
root & 0 & 0 & $\frac{1}{18}$ & $-\frac{1}{18}$ \\
left & $\frac{1}{3}$ & $-\frac{1}{3}$ & $\frac{1}{3}$ & $-\frac{1}{3}$ \\
right & $-\frac{1}{2}$ & $\frac{1}{2}$ & $-\frac{1}{2}$ & $\frac{1}{2}$ \\
\bottomrule
\end{tabular}
\end{center}
\end{table}

\end{exmp}

\section{ALTERNATIVE GFA FAMILY}\label{S7}

We have the following proposition for the $\widetilde{\operatorname{GFA}}_f(j)$ proposed in Eq.~\eqref{eq:forest-inner}.

\begin{proposition}[Asymptotic Unbiasedness]
Without loss of generality, we assume $X_1, \ldots, X_k$ are relevant features, and $X_{k+1}, \ldots, X_p$ are noisy features.
Suppose that
\begin{itemize}
    \item We are using PreDecomp as the IFA, i.e. $r_{m,j}(\bx_i) = f_{m,j}(\bx_i)$;
    \item There are infinite training samples (population case);
    \item $(X_1, \ldots, X_k), X_{k+1}, \ldots, X_p$ are mutually independent, i.e.\begin{equation}
    \begin{aligned}
    & P(X_1, X_2, \ldots, X_k, X_{k+1}, \ldots, X_p) \\
    =& P(X_1, X_2, \ldots, X_k) P(X_{k+1}) \ldots P(X_p)
    \end{aligned}
    \end{equation}
\end{itemize}
Then for any noisy feature $k + 1 \le j \le p$, we have
\begin{equation}
    \lim_{M \to \infty}\widetilde{\operatorname{GFA}}_f(j) = 0.
\end{equation}
\end{proposition}
\begin{proof}
We consider the subset of relevant features collectively as a super-feature $X_S = (X_1, X_2, \ldots, X_k)$.
The regression function has the form
\begin{equation}
\begin{aligned}
f^*(X) =& \mathbb E (Y|X) \\
=& \mathbb E (Y|X_S, X_{k+1}, \ldots, X_{p}) \\
=& \mathbb E (Y|X_S) + \underbrace{0 + \ldots + 0}_{p-k}
\end{aligned}
\end{equation}
Note that this is an additive model, so by Proposition~\ref{prop:additive}, for $k < j \le p$ we have
\begin{equation}
\begin{aligned}
\mathbb E (f_{[M], j}(X) - 0)^2 \to& 0. \\
f_{[M], j}(X) \to& 0.
\end{aligned}
\end{equation}
As a result,
\begin{equation}
\begin{aligned}
\widetilde{\operatorname{GFA}}_f(j) =& \alpha^{-1} \sum_{i=1}^N \left(\sum_{m=1}^M f_{m,j}(\bx_i)\right) y_i \\
=& \alpha^{-1} \sum_{i=1}^N f_{[M], j}(\bx_i) y_i \\
\to& \alpha^{-1} \sum_{i=1}^N 0 \cdot y_i \\
=& 0
\end{aligned}
\end{equation}
\end{proof}

As a corollary, we have shown that total gain is asymptotically unbiased in the population case.
This makes sense since the bias in split-improvement scores is essentially due to the model struggling to tell signal from noise and overfits to the latter, but access to infinite samples means there is effectively no noise.

\section{PROOFS}\label{S8}

\subsection{Proof for Proposition~\ref{prop:local-accuracy}}\label{S8.1}
\begin{proof}
\begin{equation}
\begin{aligned}
\sum_{k=1}^p f_{m,k}(X) =& \sum_{k=1}^p \sum_{t\in I_m:v(t)=k} \begin{aligned}[t]
& \left[p_m(t^\tleft)\1(X\in R_{t^{\tleft}})\right. \\
&+ p_m(t^\tright)\1(X\in R_{t^{\tright}}) \\
&- \left.p_m(t)\1(X\in R_{t})\right]\end{aligned} \\
=& \sum_{t\in I_m} \begin{aligned}[t]
& \left[p_m(t^\tleft)\1(X\in R_{t^{\tleft}})\right. \\
&+ p_m(t^\tright)\1(X\in R_{t^{\tright}}) \\
&- \left.p_m(t)\1(X\in R_{t})\right]\end{aligned}
\end{aligned}
\end{equation}
Since all inner nodes $t\in I_m$ except the root node is also the child of another node, the term $p_m(t)\1(\bx\in R_{t})$ will get canceled by the corresponding term of its parent.
The only remaining terms are the leaf nodes and the root node.
Denote $J_m$ to be the set of all leaf nodes in tree $m$, and then
\begin{equation}
\begin{aligned}
\sum_{j=1}^p f_{m,j}(X) &= \sum_{t\in J_m} p_m(t)\1(X\in R_{t}) - p_m(\text{root}) \\
f_{m}(X) + p_m(\text{root}) &= \sum_{t\in J_m} p_m(t)\1(X\in R_{t}) \\
    &= f_m(X).
\end{aligned}
\end{equation}
That completes the proof.
\end{proof}

\subsection{Proof for Proposition~\ref{prop:additive}}\label{S8.2}
\begin{proof}
First of all, we observe that with an infinite sample size and the number of trees approaching infinity, i.e., $M \to \infty$,
the GBT is consistent,
\begin{equation}
    \lim_{M\to\infty} \mathbb E (f_{[M]}(X) - f^*(X))^2 = 0,
\end{equation}
This is because gradient boosted trees are guaranteed to approach zero mean squared error for the training set and with infinite samples, i.e. when the training set is the same as the population.

Now we would like to show that for tree $m$, $f_{m,k}(X)$ is only a function of $X_k$.
In other words, with slight abuse of notations, we have $f_{m,k}(X) = f_{m,k}(X_k)$.
We prove this by induction.
When $m=0$, $f_{[m],k}(X) = 0$ which holds trivially.
Suppose the conclusion holds for $m$: $f_{[m],k}(X) = f_{[m],k}(X_k)$.
For tree $m+1$, to show $f_{m+1,k}(X)$ only depends on $X_k$, by Definition \ref{def:predecomp}, we only need to show that for each node $t$ and its parent $t_{\tparent}$, if $t_{\tparent}$ splits on feature $j=v(t_{\tparent})$, then $p_{m+1}(t) - p_{m+1}(t_{\tparent})$ is a uni-variate function of $X_j$.
Recall that in Eq.~\eqref{eq:pmt-leaf}, $p_m(t) =\frac{\sum_{\bx_i\in R_t} \left(y_i - f_{[m-1]}(\bx_i)\right)}{|R_t| + \lambda}$, but we have an infinite number of samples, so the lambda term in the denominator can be ignored, resulting in $p_{m+1}(t) = \frac{\sum_{\bx_i\in R_t} \left(y_i - f_{[m-1]}(\bx_i)\right)}{|R_t|} = \mathbb E (f^*(X)-f_{[m]}(X)|X\in R_t)$.
Since $\mathbb E(Y) = 0$, we know $p_m(\text{root}) = 0$.
By Proposition~\ref{prop:local-accuracy} and our induction assumption, $f_{[m]}(X)$ is additive, i.e.
\begin{equation}
f_{[m]}(X) = \sum_{k=1}^p f_{[m],k}(X) + 0 = \sum_{k=1}^p f_{[m],k}(X_k).
\end{equation}
Given the assumption that $f^*(X)$ is additive, we also have
\begin{equation}
f^*(X) = \sum_{k=1}^p h_k(X_k).
\end{equation}

To simplify the following deduction, let's assume $R_t = \1_{X_1 < t_1, X_2 < t_2}$. The general case follows the same logic.
Then we have
\begin{equation}
\begin{aligned}
& \mathbb E(f^*(X) - f_{[m]}(X)|X_1 < t_1, X_2 < t_2) \\
=& \sum_{k=1}^p\mathbb E(h_k(X_k) - f_{[m],k}(X_k) | X_1 < t_1, X_2 < t_2).
\end{aligned}
\end{equation}
By the independence of $X_1, \ldots, X_p$, the above equation equals to 
\begin{equation}
\begin{aligned}
& \mathbb E (h_1(X_1) - f_{[m],1}(X_1) | X_1 < t_1) \\
&+ \mathbb E (h_2(X_2) - f_{[m],2}(X_2) | X_2 < t_2) \\
&+ \underbrace{\sum_{k \ge 3}\mathbb E(h_k(X_k) - f_{[m],k}(X_k)}_{\const}.
\end{aligned}
\end{equation}
Without loss of generality, let's assume $t_{\tparent}$ is $\1_{X_1 < t_1}$.
Following the same deduction, we have
\begin{equation}
\begin{aligned}
p_{m+1}(t_{\tparent}) = \mathbb E (h_1(X_1) - f_{[m],1}(X_1) | X_1 < t_1) + \const.
\end{aligned}
\end{equation}
Therefore, we have $p_{m+1}(t) - p_{m+1}(t_{\tparent}) = \mathbb E (h_2(X_2) - f_{[m],2}(X_2) | X_2 < t_2)$ which is a uni-variable function of $X_2$.
That concludes our proof that $f_{m,k}(X)$ only depends on $X_k$.
By induction, we know $f_{[m],k}(X) = f_{[m],k}(X_k)$ for any $k=1,\ldots, p$.

Finally, we are ready to show Eq.~\eqref{Eq:prop2}.
Since $X_j$ are independent and $f_{[m],k}(X) = f_{[m],k}(X_k)$, we know the following equality:
\begin{equation}
\mathbb E (f_{[M]}(X) - f^*(X))^2 = \sum_{k=1}^p \mathbb E (f_{[M],k}(X_k) - h_k(X_k))^2.
\end{equation}

Since the left-hand side goes to zero as $M\to \infty$, the right-hand side must go to zero as well.
Therefore, for any $k$,
\begin{equation}
\mathbb E (f_{[M],k}(X_k) - h_k(X_k))^2 \to 0.
\end{equation}
That completes the proof.
\end{proof}

\subsection{Proof for Proposition~\ref{prop:gain}}\label{S8.3}

\begin{proof}
Recall the definition of $p_m(t)$ in Eq.~\eqref{eq:pmt}.
\begin{equation}
\begin{aligned}
\alpha\Delta(t) =& \alpha\frac{(\sum_{\bx_i\in R_{t^\tleft}} G_{i,m})^2}{\sum_{\bx_i\in R_{t^\tleft}} H_{i,m}+\lambda} + \alpha\frac{(\sum_{\bx_i\in R_{t^\tright}} G_{i,m})^2}{\sum_{\bx_i\in R_{t^\tright}} H_{i,m}+\lambda} \\
&- \alpha\frac{(\sum_{\bx_i\in R_{t}} G_{i,m})^2}{\sum_{\bx_i\in R_{t}} H_{i,m}+\lambda} \\
=& -p_m(t^\tleft)\sum_{i=1}^N G_{i,m}\1(\bx_i\in R_{t^\tleft}) \\
&- p_m(t^\tright)\sum_{i=1}^N G_{i,m}\1(\bx_i\in R_{t^\tright}) \\
&+ p_m(t)\sum_{i=1}^N G_{i,m}\1(\bx_i\in R_{t}) \\
=& \sum_{i=1}^N -G_{i,m} \cdot \begin{aligned}[t]
& \left[p_m(t^\tleft)\1(\bx_i\in R_{t^\tleft})\right. \\
& +p_m(t^\tright)\1(\bx_i\in R_{t^\tright}) \\
& \left.-p_m(t)\1(\bx_i\in R_{t})\right]\end{aligned}.
\end{aligned}
\end{equation}

Then, by the definition of total gain, 
\begin{equation}
\begin{aligned}
\mathrm{TotalGain}_{m,k} =& \sum_{t\in I_m:v(t)=k}\Delta(t)\\
=& \alpha^{-1} \sum_{m=1}^M\sum_{t\in I_m:v(t)=k} \sum_{i=1}^N -G_{i,m} \\
& \times \begin{aligned}[t]
& \left[p_m(t^\tleft)\1(\bx_i\in R_{t^\tleft})\right. \\
&+ p_m(t^\tright)\1(\bx_i\in R_{t^\tright}) \\
&- \left.p_m(t)\1(\bx_i\in R_{t})\right]\end{aligned}\\
=& -\alpha^{-1}\sum_{i=1}^N f_{m,k}(\bx_i)G_{i,m}.
\end{aligned}
\end{equation}
That completes the proof.
\end{proof}

\subsection{Proof for Remark~\ref{remark:sum-to-zero}}\label{S8.4}

\begin{proof}
By a similar argument as the proof for Proposition~\ref{prop:local-accuracy}, we have
\begin{equation}
\begin{aligned}
& \sum_{\bx \in R_\text{root}} f_m(\bx) \\
=& \sum_{\bx \in R_\text{root}} \sum_{t\in I_m:v(t)=k} \begin{aligned}[t]
& \left[p_m(t^\tleft)\1(\bx\in R_{t^{\tleft}})\right. \\
&+ p_m(t^\tright)\1(\bx\in R_{t^{\tright}}) \\
&- \left.p_m(t)\1(\bx\in R_{t})\right] \\
\end{aligned} \\
=& \sum_{t\in I_m:v(t)=k} \begin{aligned}[t]
& \left(p_m(t^\tleft)|R_{t^{\tleft}}|\right. \\
&+ p_m(t^\tright)|R_{t^{\tright}}| \\
&- \left.p_m(t)|R_{t}|\right) \\
\end{aligned} \\
=& \sum_{t\in I_m:v(t)=k} \begin{aligned}[t]
& \left(\sum_{\bx_i\in R_\tleft} \left(y_i - f_{[m-1]}(\bx_i)\right)\right. \\
&+ \sum_{\bx_i\in R_\tright} \left(y_i - f_{[m-1]}(\bx_i)\right) \\
&- \left.\sum_{\bx_i\in R_t} \left(y_i - f_{[m-1]}(\bx_i)\right)\right) \\
\end{aligned} \\
=& \sum_{t\in I_m:v(t)=k} 0 \\
=& 0
\end{aligned}
\end{equation}
\end{proof}

\section{ADDITIONAL EXPERIMENTAL RESULTS}\label{S9}

In this section, we provide experimental results which do not fit within the main text.

\subsection{AUC Scores}\label{S9.1}

\begin{longtable}{lllllrrrr}
\caption{AUC Scores for Model Trained with Standard Hyperparameters}
\label{table:standard-auc}\\
\toprule
          &            &           &       &          &  AUC\_mean &  AUC\_std &  risk\_mean &  risk\_std \\
Dataset & Task & GFA & Domain & IFA &           &          &            &           \\
\midrule
\endfirsthead
\caption[]{AUC Scores for Model Trained with Standard Hyperparameters} \\
\toprule
          &            &           &       &          &  AUC\_mean &  AUC\_std &  risk\_mean &  risk\_std \\
Dataset & Task & GFA & Domain & IFA &           &          &            &           \\
\midrule
\endhead
\midrule
\multicolumn{9}{r}{{Continued on next page}} \\
\midrule
\endfoot

\bottomrule
\endlastfoot
\multirow{28}{*}{ChIP} & \multirow{14}{*}{Classification} & \multirow{4}{*}{Abs} & \multirow{2}{*}{Train} & PreDecomp &    0.8581 &   0.1200 &     0.4673 &    0.0168 \\
          &            &           &       & TreeSHAP &    0.8651 &   0.1181 &     0.4673 &    0.0168 \\
\cline{4-9}
          &            &           & \multirow{2}{*}{Valid} & PreDecomp &    0.8599 &   0.1172 &     0.4673 &    0.0168 \\
          &            &           &       & TreeSHAP &    0.8643 &   0.1190 &     0.4673 &    0.0168 \\
\cline{3-9}
\cline{4-9}
          &            & \multirow{4}{*}{ForestInner} & \multirow{2}{*}{Train} & PreDecomp &    0.7953 &   0.1207 &     0.4673 &    0.0168 \\
          &            &           &       & TreeSHAP &    0.8640 &   0.1060 &     0.4673 &    0.0168 \\
\cline{4-9}
          &            &           & \multirow{2}{*}{Valid} & PreDecomp &    0.7975 &   0.1619 &     0.4673 &    0.0168 \\
          &            &           &       & TreeSHAP &    0.7553 &   0.1727 &     0.4673 &    0.0168 \\
\cline{3-9}
\cline{4-9}
          &            & \multirow{2}{*}{Permutation} & Train & Permutation &    0.8367 &   0.1091 &     0.4673 &    0.0168 \\
          &            &           & Valid & Permutation &    0.8275 &   0.1245 &     0.4673 &    0.0168 \\
\cline{3-9}
          &            & \multirow{4}{*}{TreeInner} & \multirow{2}{*}{Train} & PreDecomp &    0.8043 &   0.1213 &     0.4673 &    0.0168 \\
          &            &           &       & TreeSHAP &    0.8645 &   0.1081 &     0.4673 &    0.0168 \\
\cline{4-9}
          &            &           & \multirow{2}{*}{Valid} & PreDecomp &    0.8671 &   0.1071 &     0.4673 &    0.0168 \\
          &            &           &       & TreeSHAP &    0.8756 &   0.1021 &     0.4673 &    0.0168 \\
\cline{2-9}
\cline{3-9}
\cline{4-9}
          & \multirow{14}{*}{Regression} & \multirow{4}{*}{Abs} & \multirow{2}{*}{Train} & PreDecomp &    0.9399 &   0.0681 &     0.1504 &    0.0867 \\
          &            &           &       & TreeSHAP &    0.9413 &   0.0688 &     0.1504 &    0.0867 \\
\cline{4-9}
          &            &           & \multirow{2}{*}{Valid} & PreDecomp &    0.9383 &   0.0693 &     0.1504 &    0.0867 \\
          &            &           &       & TreeSHAP &    0.9405 &   0.0696 &     0.1504 &    0.0867 \\
\cline{3-9}
\cline{4-9}
          &            & \multirow{4}{*}{ForestInner} & \multirow{2}{*}{Train} & PreDecomp &    0.9172 &   0.0663 &     0.1504 &    0.0867 \\
          &            &           &       & TreeSHAP &    0.9461 &   0.0651 &     0.1504 &    0.0867 \\
\cline{4-9}
          &            &           & \multirow{2}{*}{Valid} & PreDecomp &    0.9229 &   0.0917 &     0.1504 &    0.0867 \\
          &            &           &       & TreeSHAP &    0.8867 &   0.0837 &     0.1504 &    0.0867 \\
\cline{3-9}
\cline{4-9}
          &            & \multirow{2}{*}{Permutation} & Train & Permutation &    0.9369 &   0.0725 &     0.1504 &    0.0867 \\
          &            &           & Valid & Permutation &    0.9472 &   0.0587 &     0.1504 &    0.0867 \\
\cline{3-9}
          &            & \multirow{4}{*}{TreeInner} & \multirow{2}{*}{Train} & PreDecomp &    0.9235 &   0.0694 &     0.1504 &    0.0867 \\
          &            &           &       & TreeSHAP &    0.9472 &   0.0661 &     0.1504 &    0.0867 \\
\cline{4-9}
          &            &           & \multirow{2}{*}{Valid} & PreDecomp &    0.9604 &   0.0618 &     0.1504 &    0.0867 \\
          &            &           &       & TreeSHAP &    0.9528 &   0.0734 &     0.1504 &    0.0867 \\
\cline{1-9}
\cline{2-9}
\cline{3-9}
\cline{4-9}
\multirow{28}{*}{Simulated} & \multirow{14}{*}{Classification} & \multirow{4}{*}{Abs} & \multirow{2}{*}{Train} & PreDecomp &    0.6564 &   0.1334 &     0.4745 &    0.0191 \\
          &            &           &       & TreeSHAP &    0.6749 &   0.1356 &     0.4745 &    0.0191 \\
\cline{4-9}
          &            &           & \multirow{2}{*}{Valid} & PreDecomp &    0.6500 &   0.1364 &     0.4745 &    0.0191 \\
          &            &           &       & TreeSHAP &    0.6731 &   0.1361 &     0.4745 &    0.0191 \\
\cline{3-9}
\cline{4-9}
          &            & \multirow{4}{*}{ForestInner} & \multirow{2}{*}{Train} & PreDecomp &    0.4876 &   0.1379 &     0.4745 &    0.0191 \\
          &            &           &       & TreeSHAP &    0.6764 &   0.1337 &     0.4745 &    0.0191 \\
\cline{4-9}
          &            &           & \multirow{2}{*}{Valid} & PreDecomp &    0.7722 &   0.1254 &     0.4745 &    0.0191 \\
          &            &           &       & TreeSHAP &    0.7191 &   0.1438 &     0.4745 &    0.0191 \\
\cline{3-9}
\cline{4-9}
          &            & \multirow{2}{*}{Permutation} & Train & Permutation &    0.6712 &   0.1386 &     0.4745 &    0.0191 \\
          &            &           & Valid & Permutation &    0.6599 &   0.1097 &     0.4745 &    0.0191 \\
\cline{3-9}
          &            & \multirow{4}{*}{TreeInner} & \multirow{2}{*}{Train} & PreDecomp &    0.5091 &   0.1429 &     0.4745 &    0.0191 \\
          &            &           &       & TreeSHAP &    0.6798 &   0.1356 &     0.4745 &    0.0191 \\
\cline{4-9}
          &            &           & \multirow{2}{*}{Valid} & PreDecomp &    0.7856 &   0.1277 &     0.4745 &    0.0191 \\
          &            &           &       & TreeSHAP &    0.7520 &   0.1541 &     0.4745 &    0.0191 \\
\cline{2-9}
\cline{3-9}
\cline{4-9}
          & \multirow{14}{*}{Regression} & \multirow{4}{*}{Abs} & \multirow{2}{*}{Train} & PreDecomp &    0.3533 &   0.1853 &     6.8457 &    1.6117 \\
          &            &           &       & TreeSHAP &    0.3776 &   0.1983 &     6.8457 &    1.6117 \\
\cline{4-9}
          &            &           & \multirow{2}{*}{Valid} & PreDecomp &    0.3513 &   0.1849 &     6.8457 &    1.6117 \\
          &            &           &       & TreeSHAP &    0.3773 &   0.1988 &     6.8457 &    1.6117 \\
\cline{3-9}
\cline{4-9}
          &            & \multirow{4}{*}{ForestInner} & \multirow{2}{*}{Train} & PreDecomp &    0.2147 &   0.1183 &     6.8457 &    1.6117 \\
          &            &           &       & TreeSHAP &    0.3651 &   0.1993 &     6.8457 &    1.6117 \\
\cline{4-9}
          &            &           & \multirow{2}{*}{Valid} & PreDecomp &    0.6380 &   0.0976 &     6.8457 &    1.6117 \\
          &            &           &       & TreeSHAP &    0.6429 &   0.0965 &     6.8457 &    1.6117 \\
\cline{3-9}
\cline{4-9}
          &            & \multirow{2}{*}{Permutation} & Train & Permutation &    0.3769 &   0.2002 &     6.8457 &    1.6117 \\
          &            &           & Valid & Permutation &    0.5493 &   0.1392 &     6.8457 &    1.6117 \\
\cline{3-9}
          &            & \multirow{4}{*}{TreeInner} & \multirow{2}{*}{Train} & PreDecomp &    0.2278 &   0.1272 &     6.8457 &    1.6117 \\
          &            &           &       & TreeSHAP &    0.3673 &   0.1989 &     6.8457 &    1.6117 \\
\cline{4-9}
          &            &           & \multirow{2}{*}{Valid} & PreDecomp &    0.6384 &   0.0953 &     6.8457 &    1.6117 \\
          &            &           &       & TreeSHAP &    0.6496 &   0.1015 &     6.8457 &    1.6117 \\
\end{longtable}

\subsection{Sweeping over \texttt{eta}}

\begin{figure}[ht]
\vspace{.3in}
\includegraphics[width=\textwidth]{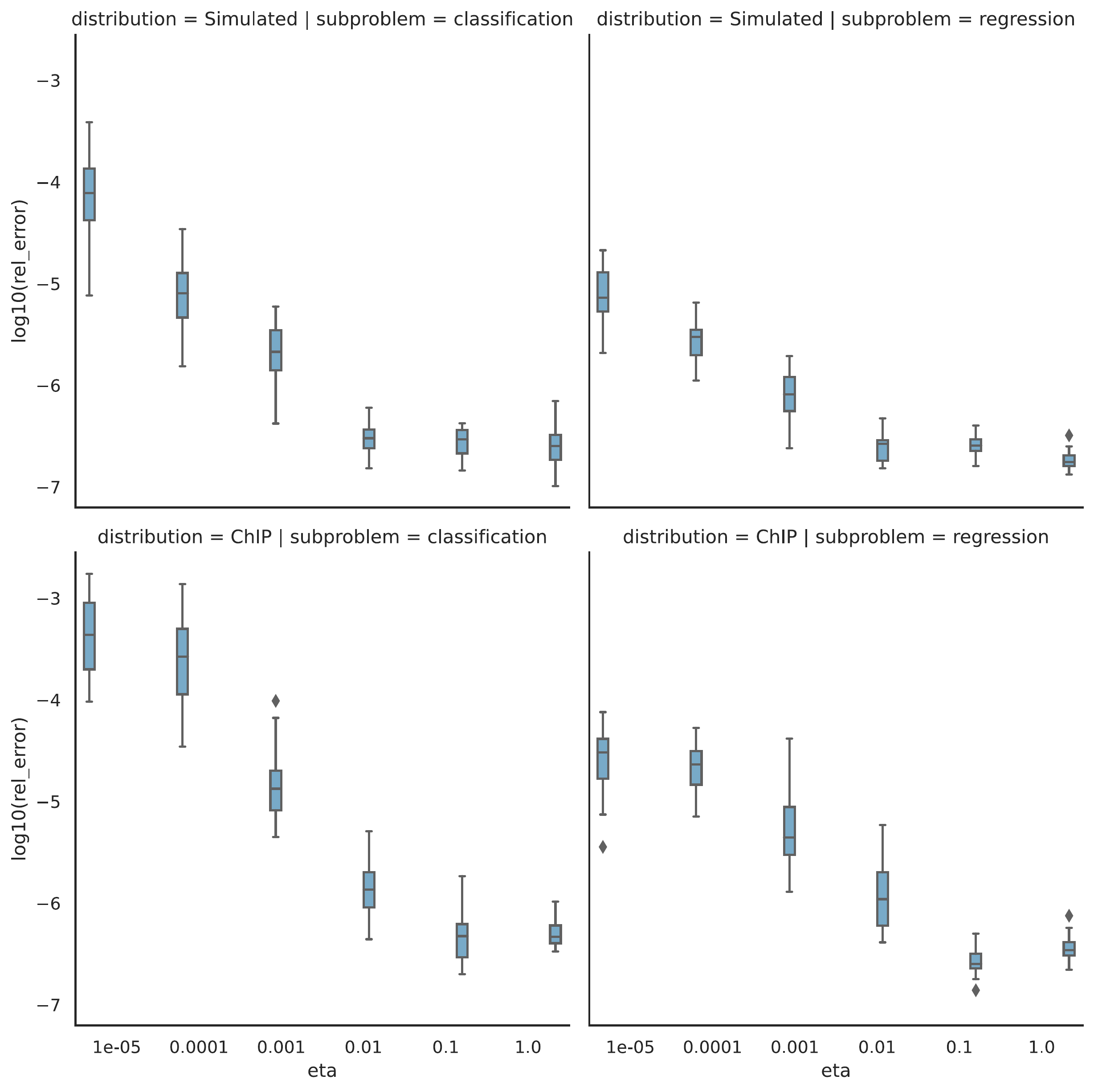}
\vspace{.3in}
\caption{Logarithm of the maximum absolute difference between the normalized total gain calculated with our methodology and the built-in method.}
\label{fig:error-eta}
\end{figure}

\begin{figure}[ht]
\vspace{.3in}
\includegraphics[width=\textwidth]{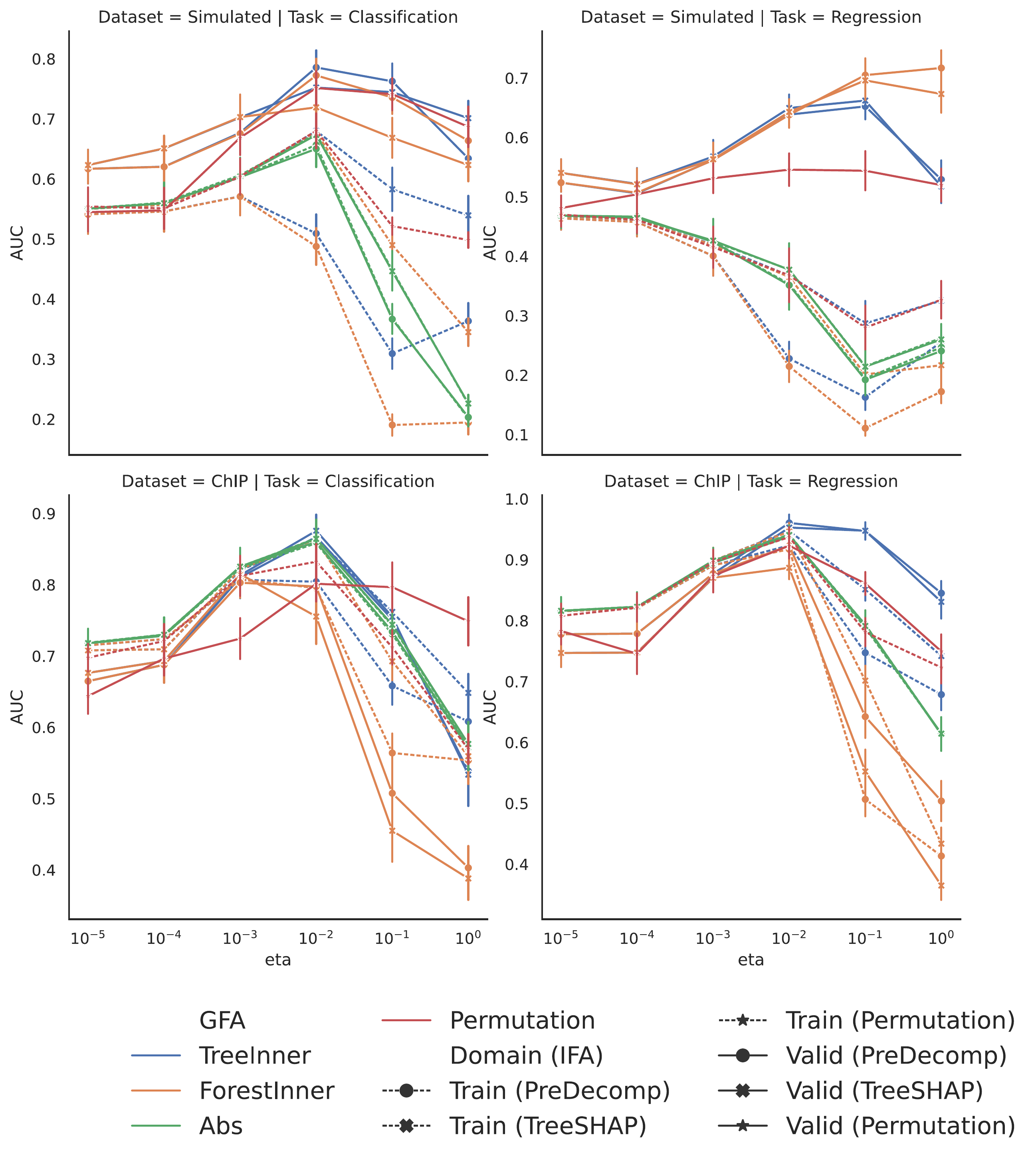}
\vspace{.3in}
\caption{
AUC score for noisy feature identification, averaged across 20 replications.
The error bars correspond to one standard error.
Higher is better.
}
\label{fig:auc-eta}
\end{figure}

\begin{figure}[ht]
\vspace{.3in}
\includegraphics[width=\textwidth]{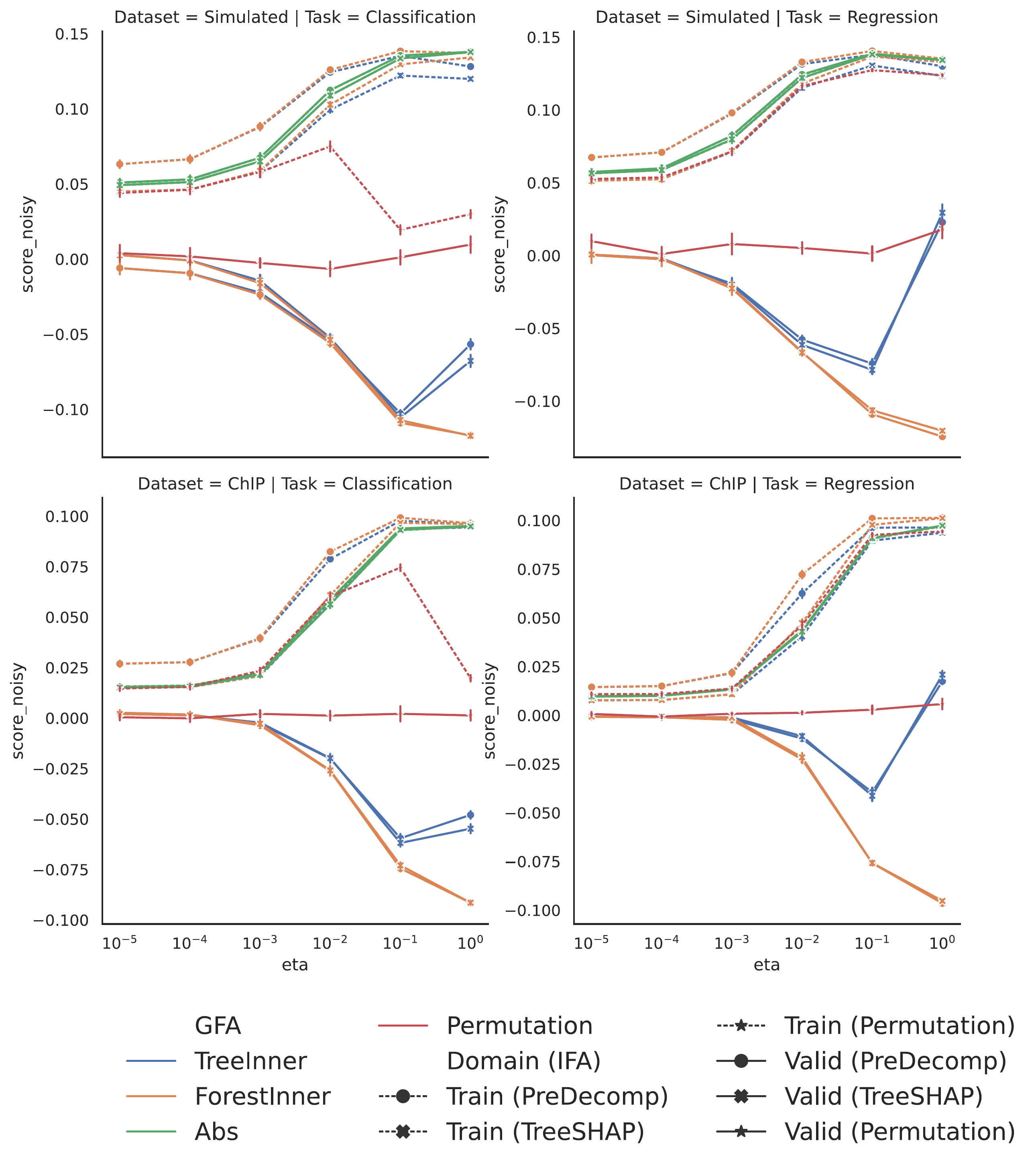}
\vspace{.3in}
\caption{
The average normalized score for the noisy features, averaged across 20 replications.
The error bars correspond to one standard error.
Lower is better.
}
\label{fig:score-noisy-eta}
\end{figure}

\begin{figure}[ht]
\vspace{.3in}
\includegraphics[width=\textwidth]{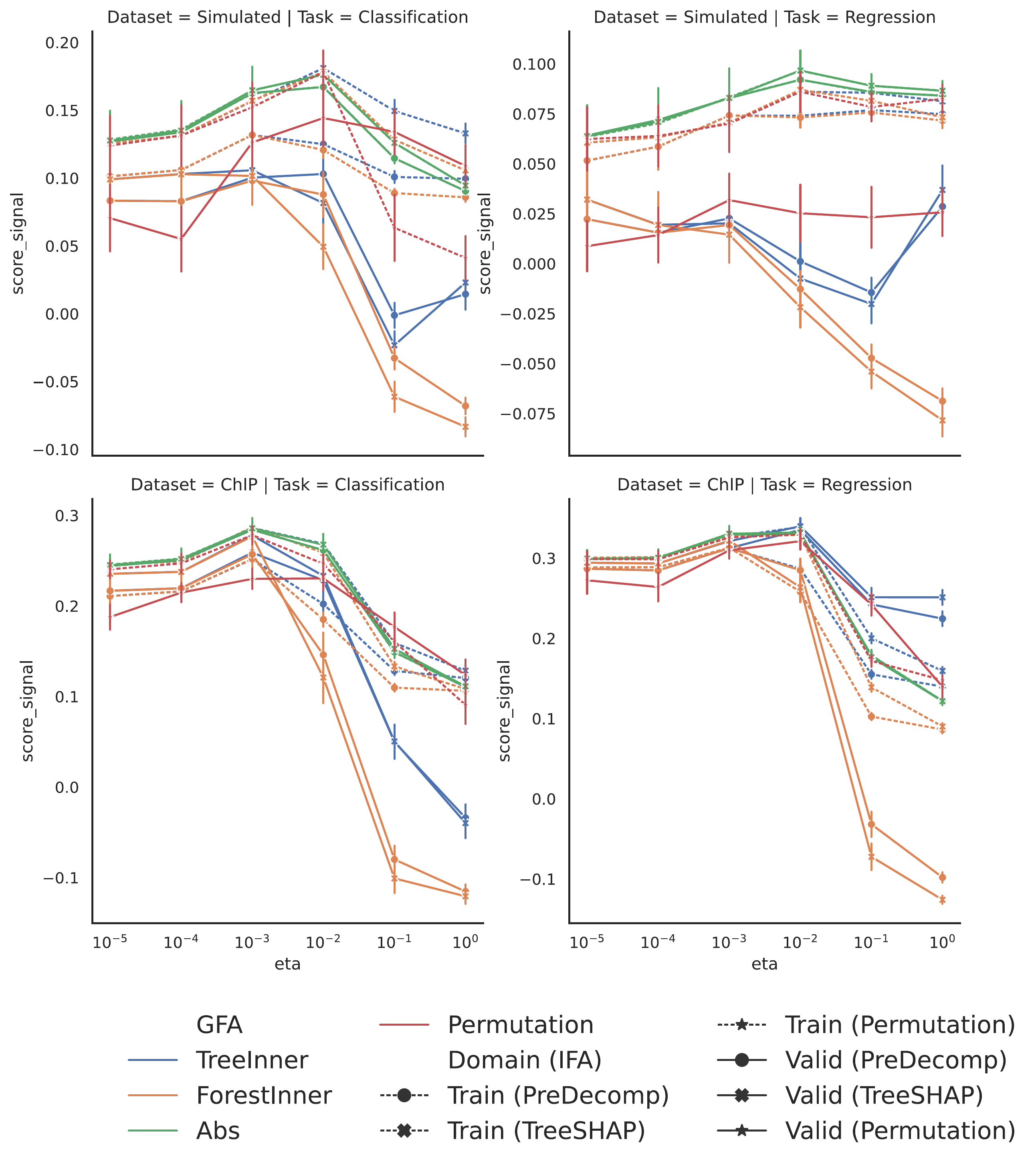}
\vspace{.3in}
\caption{
The average normalized score for the relevant features, averaged across 20 replications.
The error bars correspond to one standard error.
Lower is better.
}
\label{fig:score-signal-eta}
\end{figure}

\begin{figure}[ht]
\vspace{.3in}
\includegraphics[width=\textwidth]{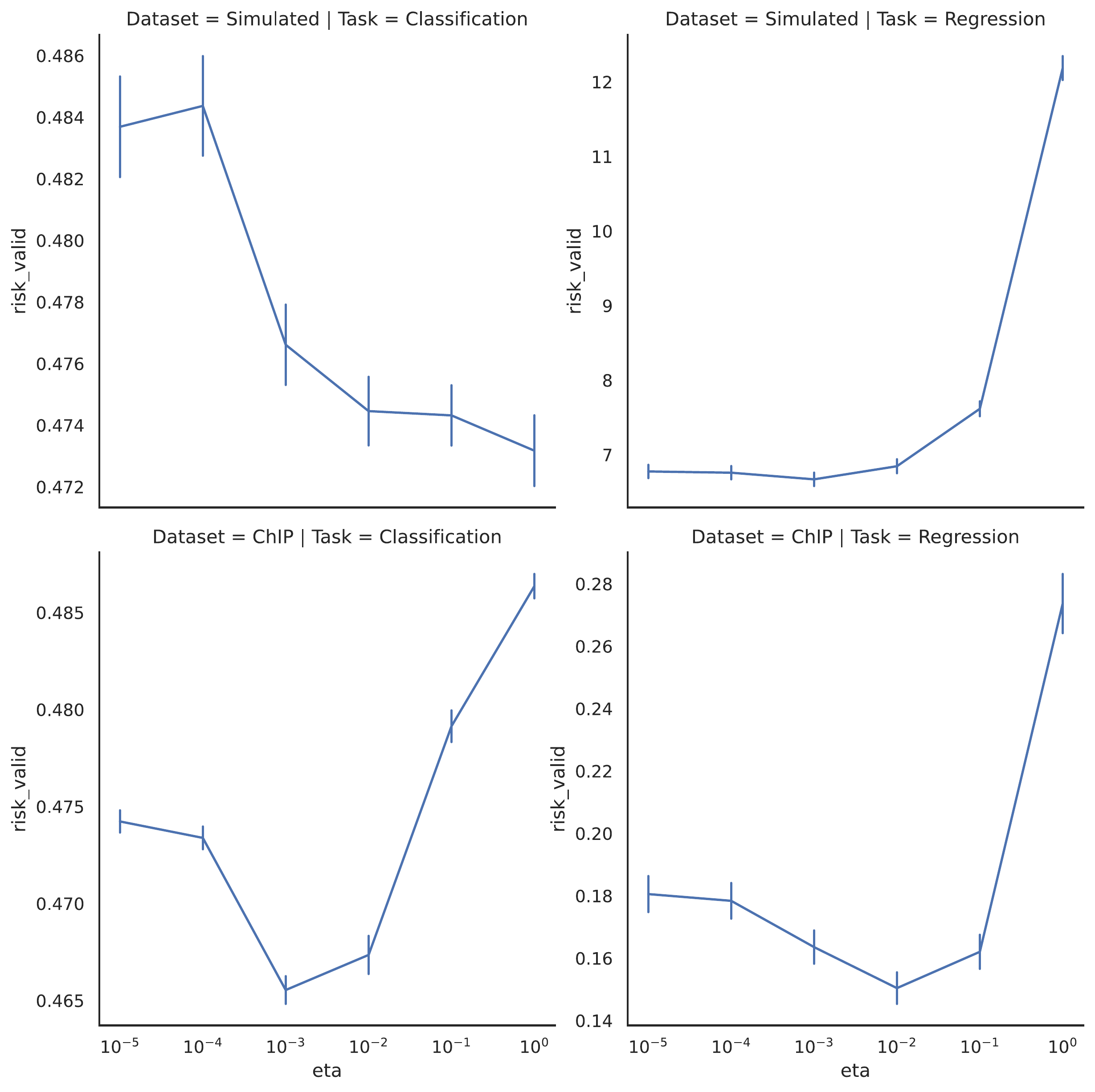}
\vspace{.3in}
\caption{
Risk evaluated on the validation set.
The error bars correspond to one standard error.
Lower is better.
}
\label{fig:risk-eta}
\end{figure}

\subsection{Sweeping over \texttt{max\_depth}}

\begin{figure}[ht]
\vspace{.3in}
\includegraphics[width=\textwidth]{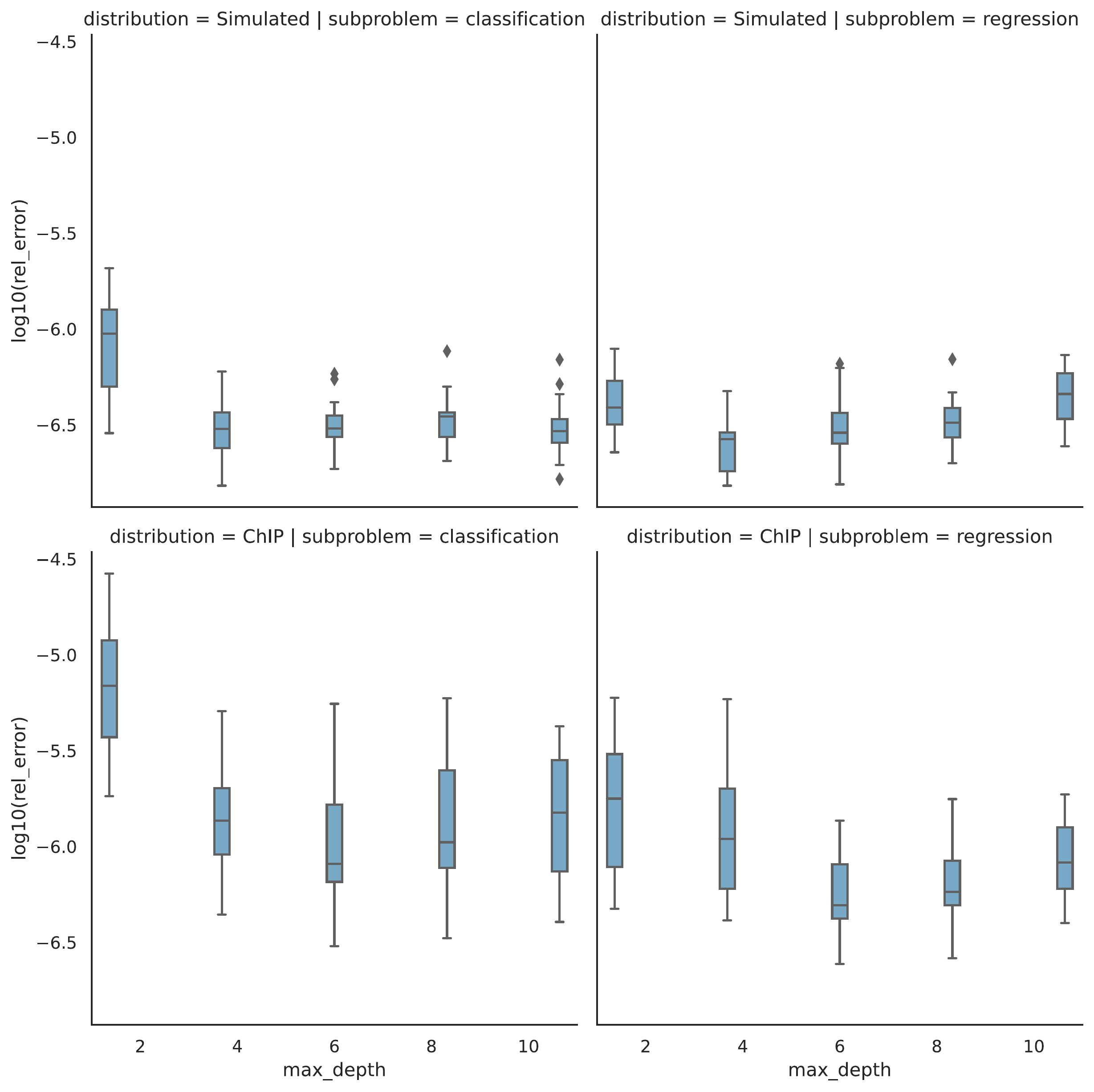}
\vspace{.3in}
\caption{Logarithm of the maximum absolute difference between the normalized total gain calculated with our methodology and the built-in method.}
\label{fig:error-max_depth}
\end{figure}

\begin{figure}[ht]
\vspace{.3in}
\includegraphics[width=\textwidth]{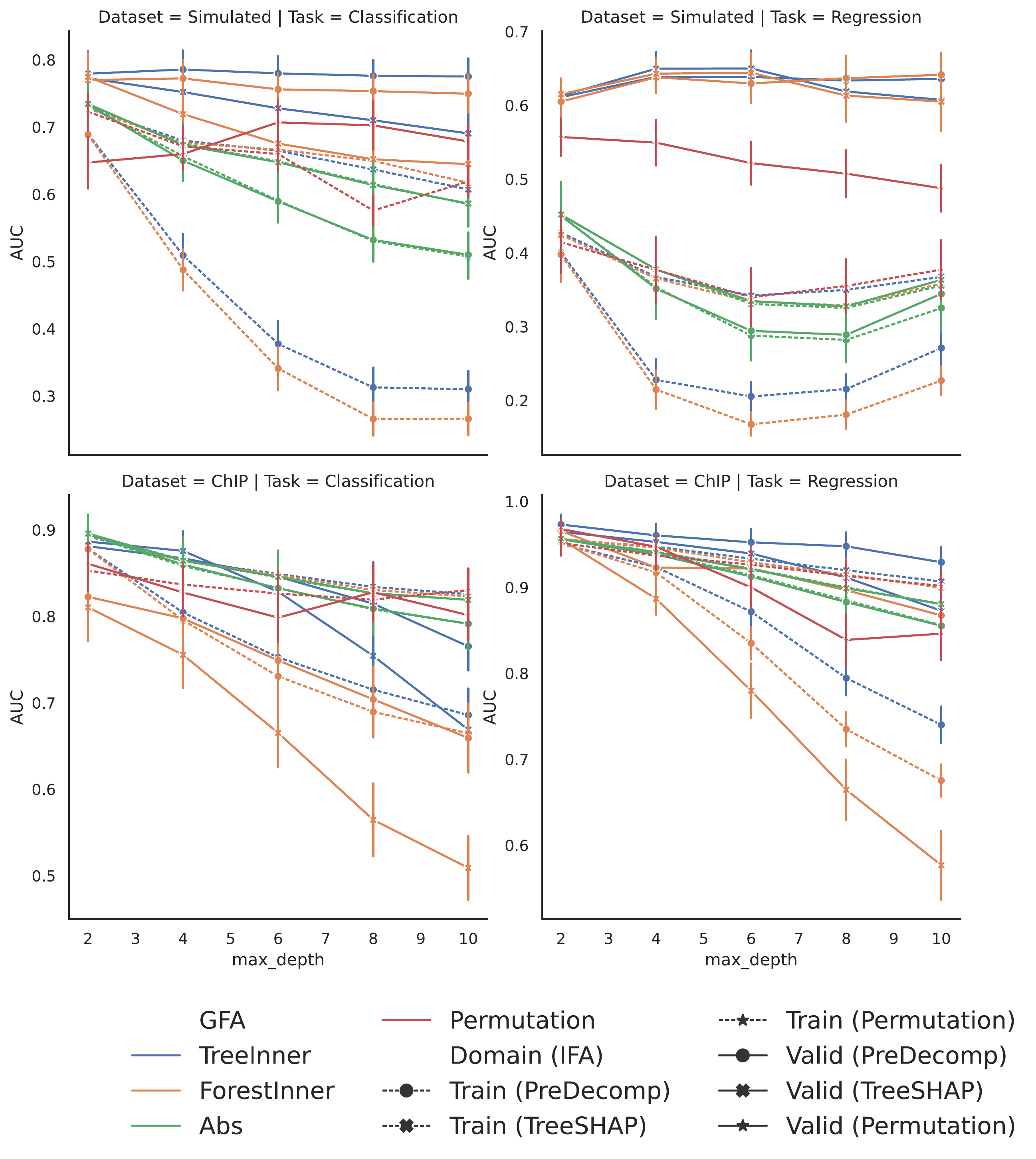}
\vspace{.3in}
\caption{
AUC score for noisy feature identification, averaged across 20 replications.
The error bars correspond to one standard error.
Higher is better.
}
\label{fig:auc-max_depth}
\end{figure}

\begin{figure}[ht]
\vspace{.3in}
\includegraphics[width=\textwidth]{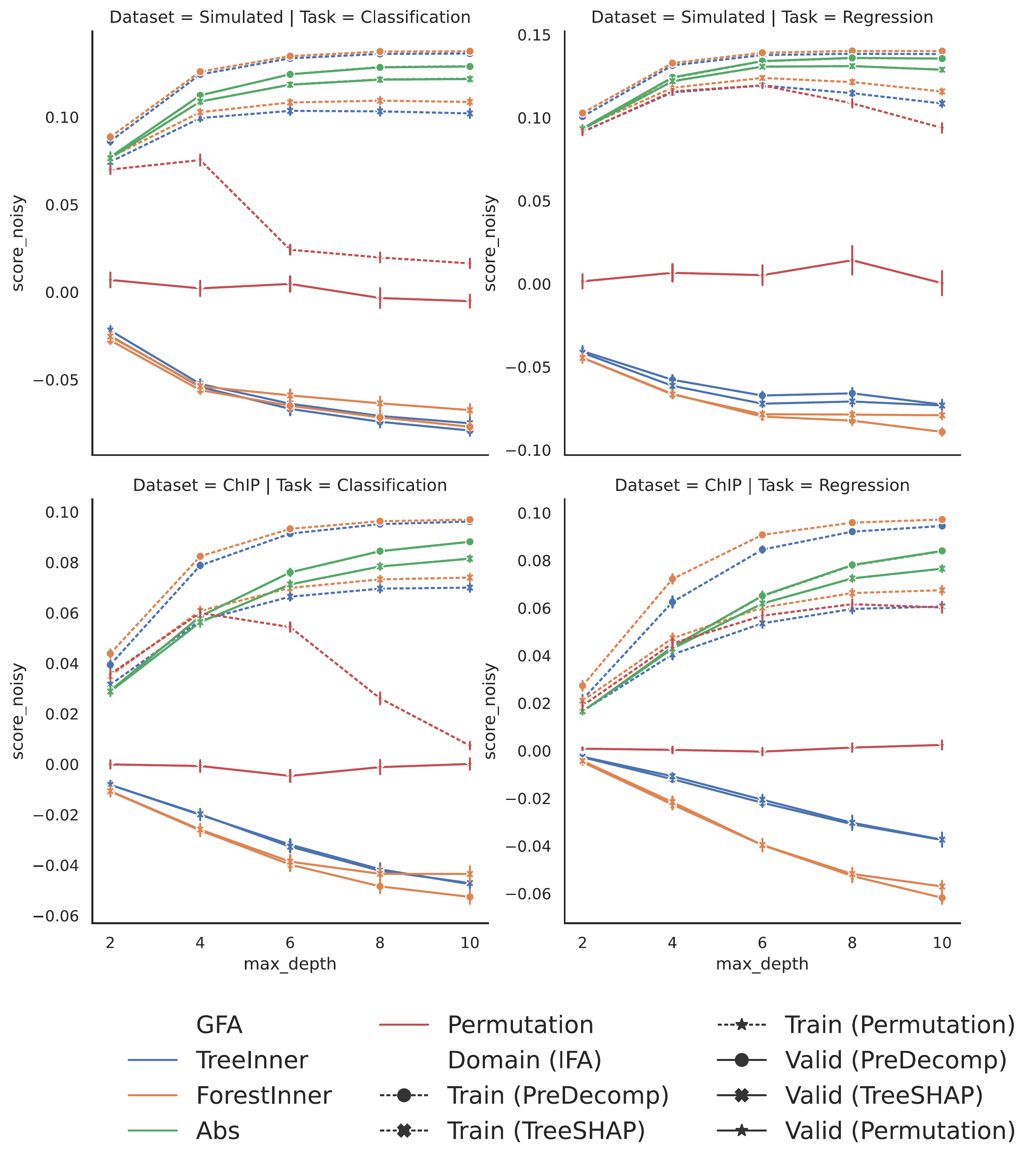}
\vspace{.3in}
\caption{
The average normalized score for the noisy features, averaged across 20 replications.
The error bars correspond to one standard error.
Lower is better.
}
\label{fig:score-noisy-max_depth}
\end{figure}

\begin{figure}[ht]
\vspace{.3in}
\includegraphics[width=\textwidth]{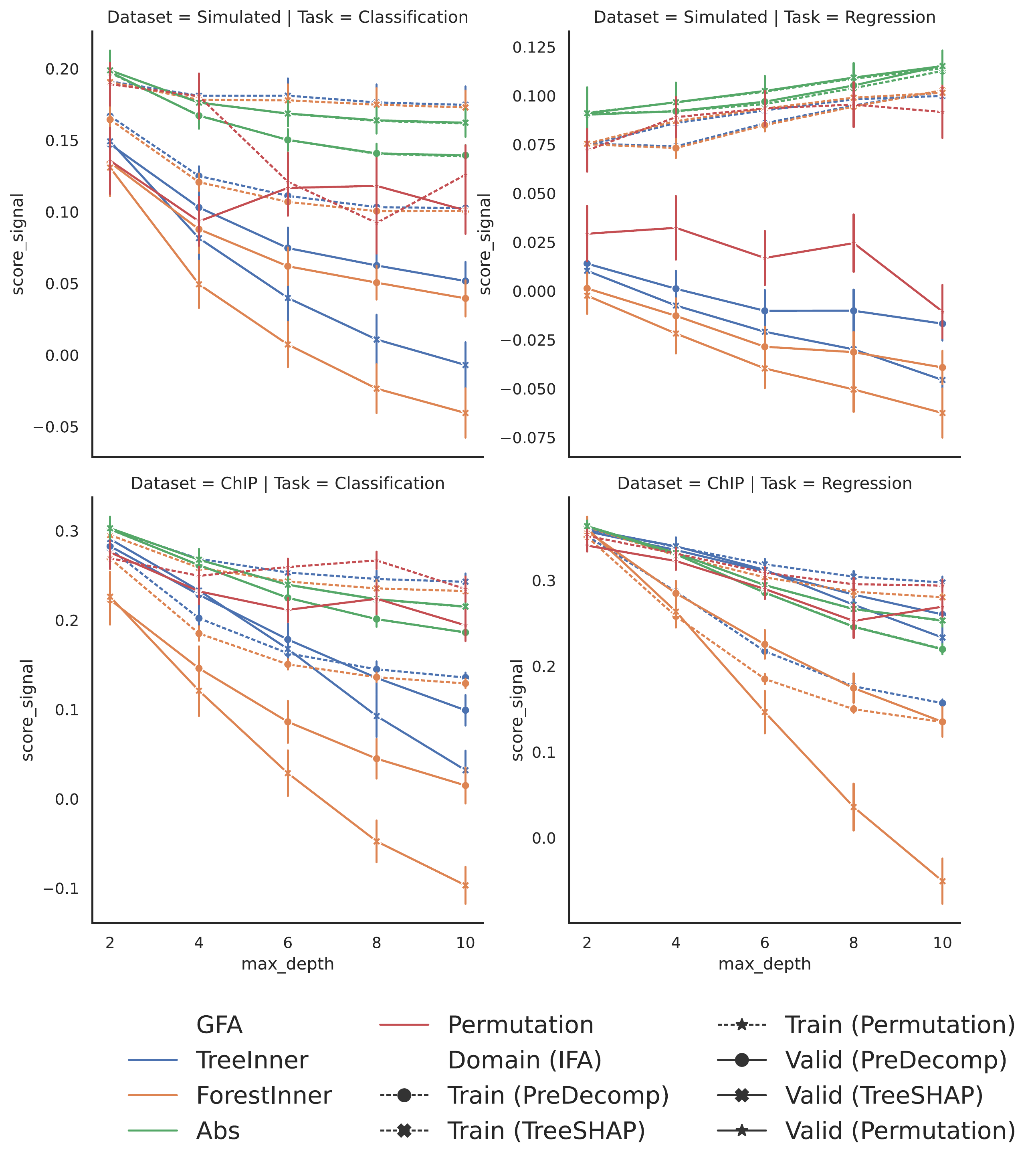}
\vspace{.3in}
\caption{
The average normalized score for the relevant features, averaged across 20 replications.
The error bars correspond to one standard error.
Lower is better.
}
\label{fig:score-signal-max_depth}
\end{figure}

\begin{figure}[ht]
\vspace{.3in}
\includegraphics[width=\textwidth]{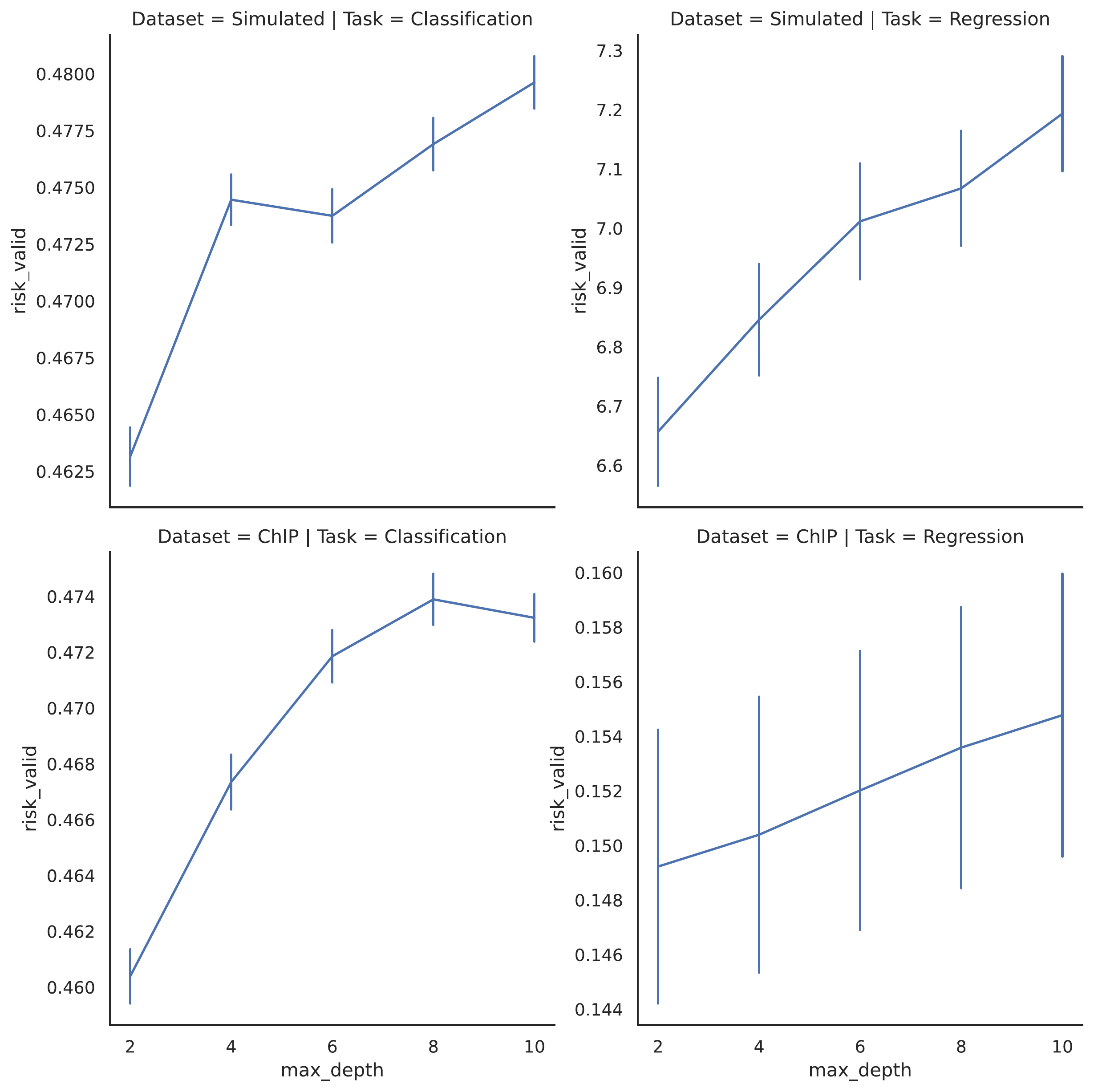}
\vspace{.3in}
\caption{
Risk evaluated on the validation set.
The error bars correspond to one standard error.
Lower is better.
}
\label{fig:risk-max_depth}
\end{figure}

\subsection{Sweeping over \texttt{min\_child\_weight}}

\begin{figure}[ht]
\vspace{.3in}
\includegraphics[width=\textwidth]{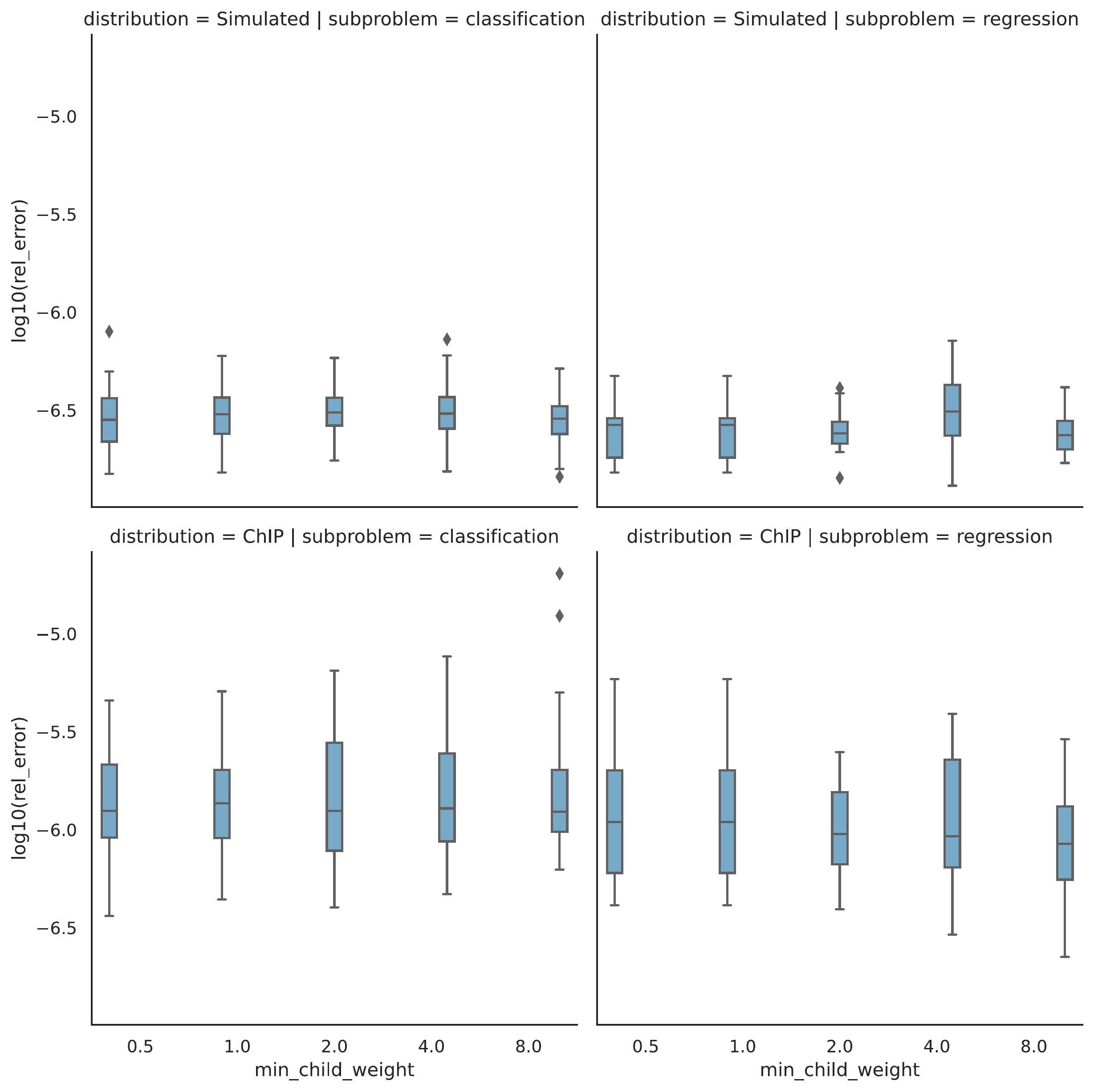}
\vspace{.3in}
\caption{Logarithm of the maximum absolute difference between the normalized total gain calculated with our methodology and the built-in method.}
\label{fig:error-min_child_weight}
\end{figure}

\begin{figure}[ht]
\vspace{.3in}
\includegraphics[width=\textwidth]{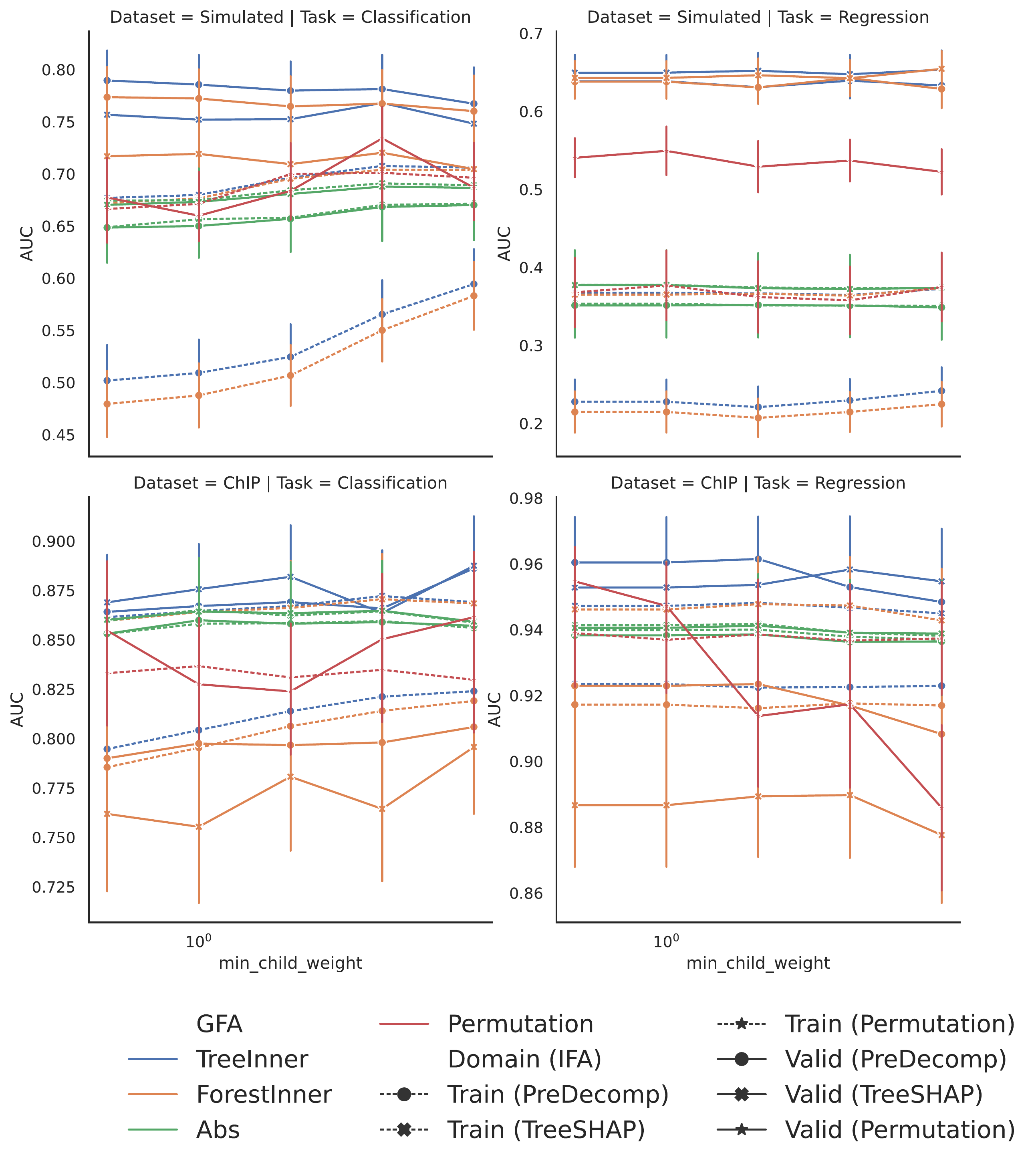}
\vspace{.3in}
\caption{
AUC score for noisy feature identification, averaged across 20 replications.
The error bars correspond to one standard error.
Higher is better.
}
\label{fig:auc-min_child_weight}
\end{figure}

\begin{figure}[ht]
\vspace{.3in}
\includegraphics[width=\textwidth]{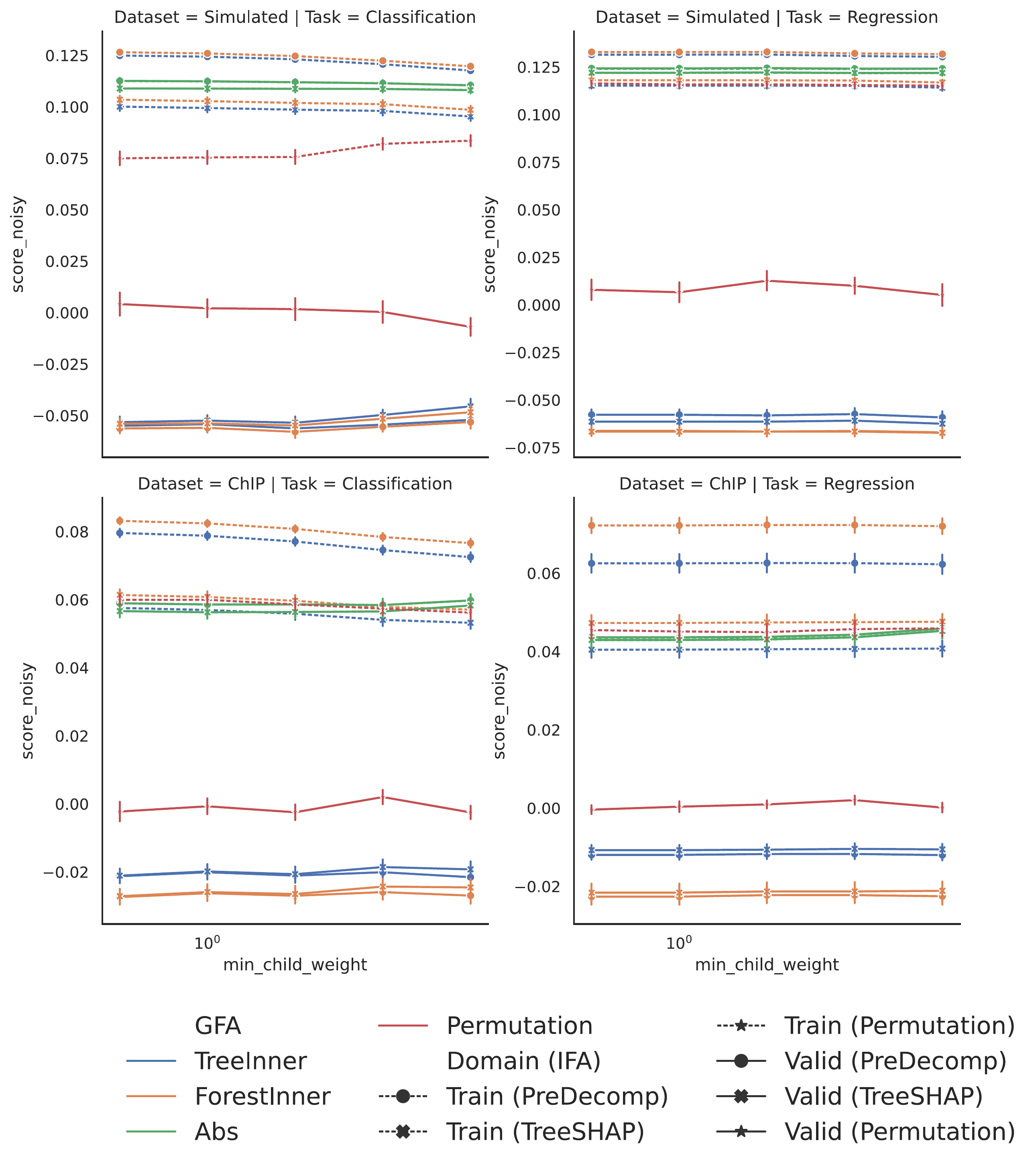}
\vspace{.3in}
\caption{
The average normalized score for the noisy features, averaged across 20 replications.
The error bars correspond to one standard error.
Lower is better.
}
\label{fig:score-noisy-min_child_weight}
\end{figure}

\begin{figure}[ht]
\vspace{.3in}
\includegraphics[width=\textwidth]{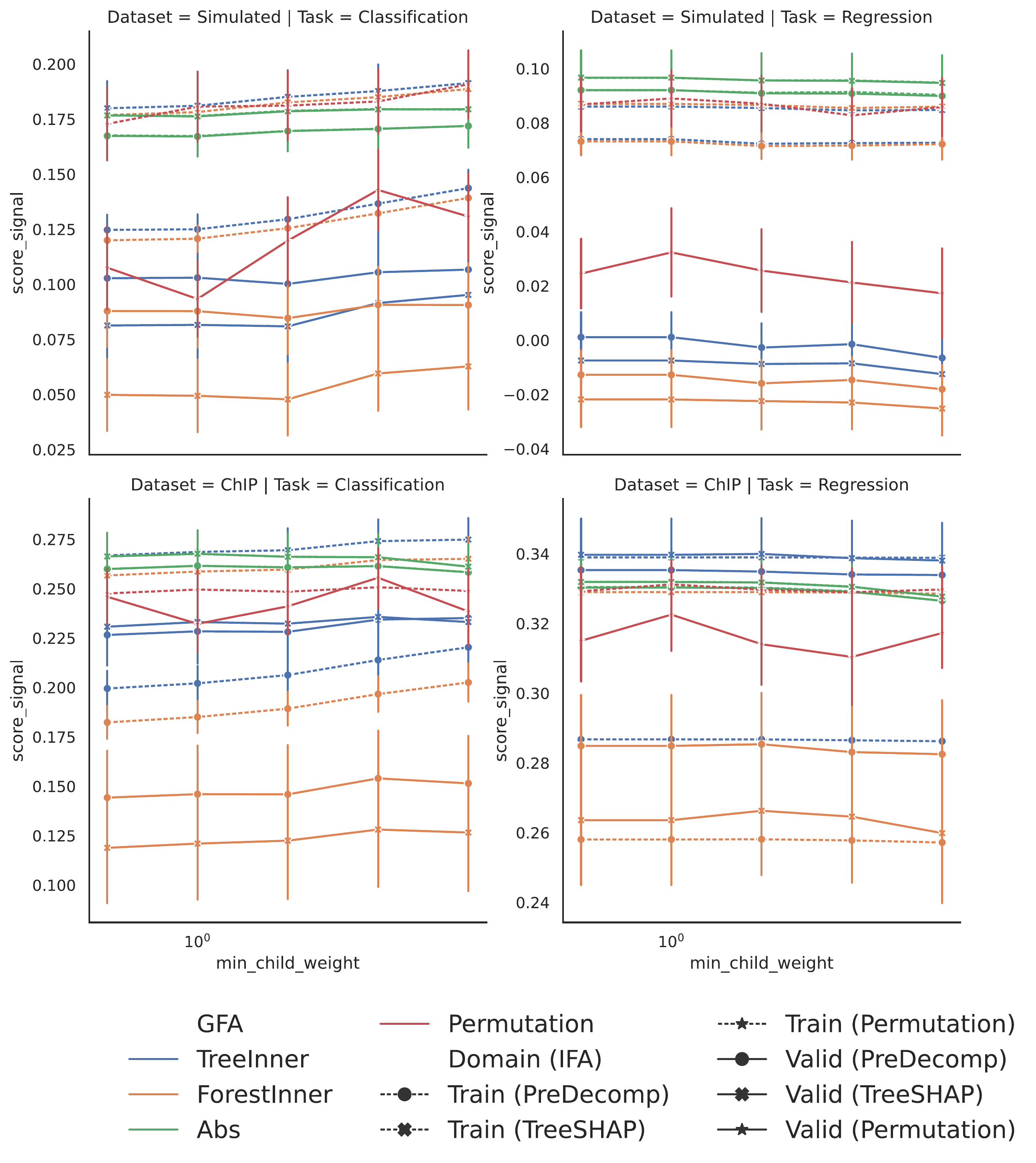}
\vspace{.3in}
\caption{
The average normalized score for the relevant features, averaged across 20 replications.
The error bars correspond to one standard error.
Lower is better.
}
\label{fig:score-signal-min_child_weight}
\end{figure}

\begin{figure}[ht]
\vspace{.3in}
\includegraphics[width=\textwidth]{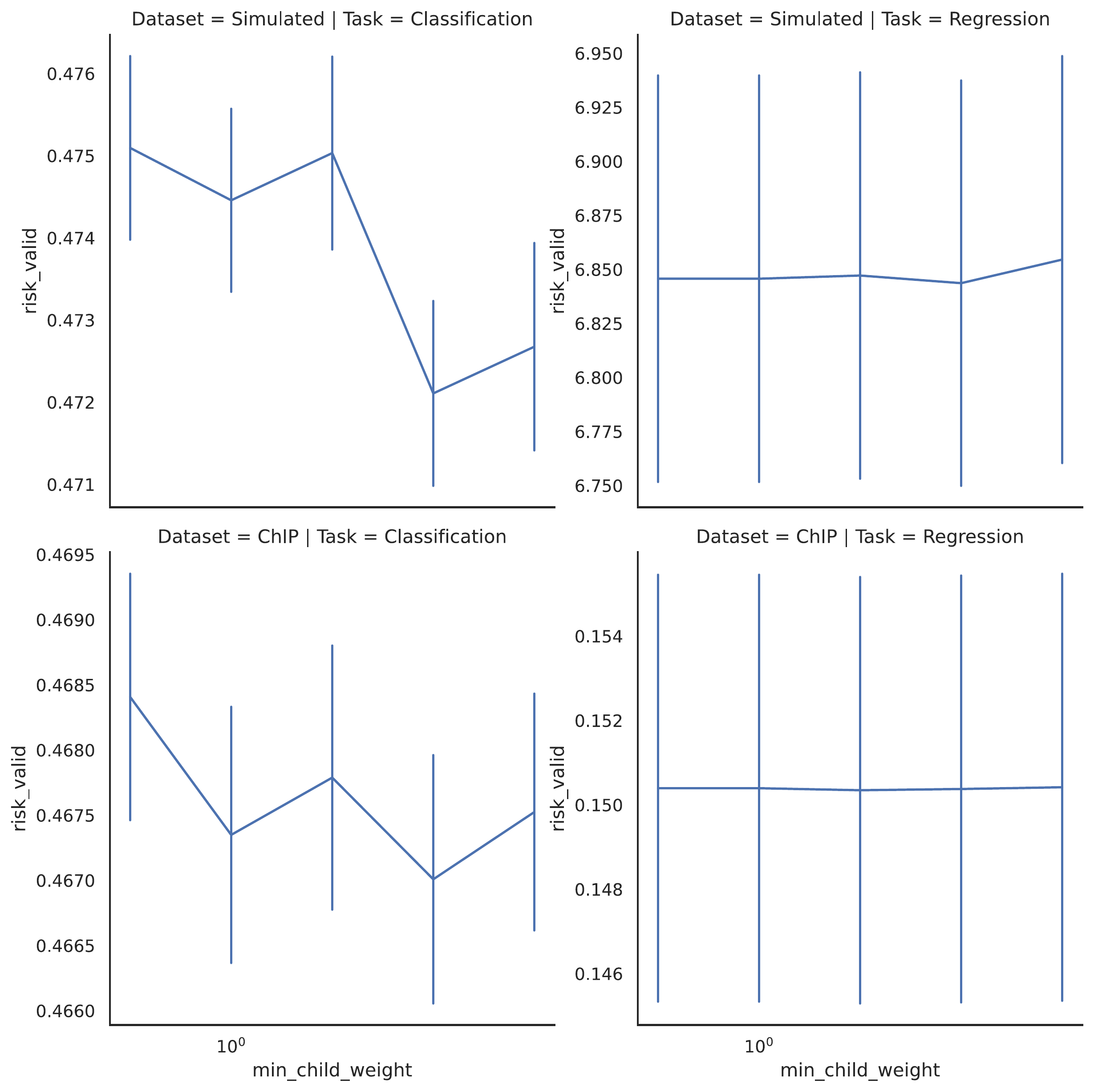}
\vspace{.3in}
\caption{
Risk evaluated on the validation set.
The error bars correspond to one standard error.
Lower is better.
}
\label{fig:risk-min_child_weight}
\end{figure}

\subsection{Sweeping over \texttt{num\_boost\_round}}

\begin{figure}[ht]
\vspace{.3in}
\includegraphics[width=\textwidth]{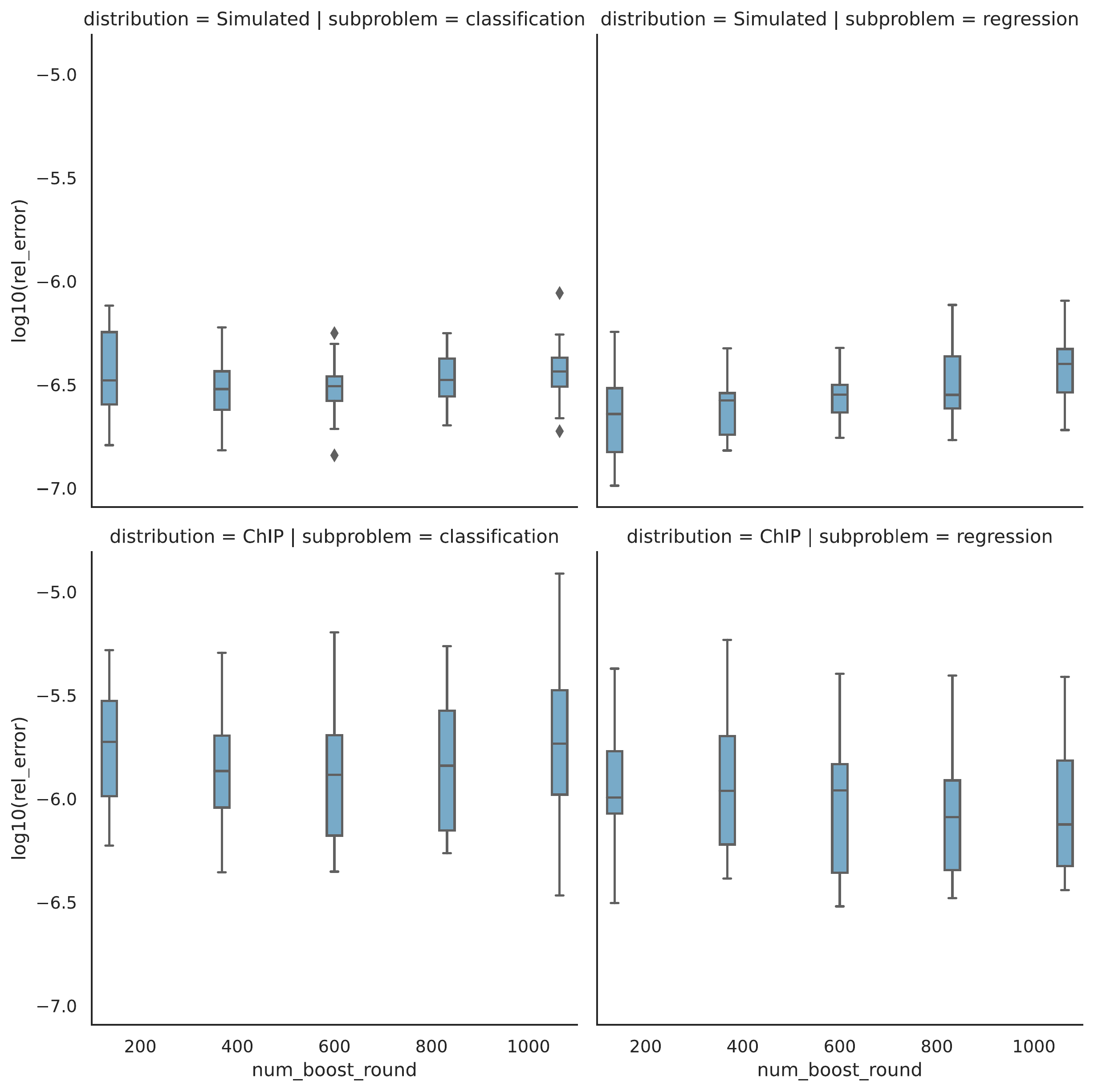}
\vspace{.3in}
\caption{Logarithm of the maximum absolute difference between the normalized total gain calculated with our methodology and the built-in method.}
\label{fig:error-num_boost_round}
\end{figure}

\begin{figure}[ht]
\vspace{.3in}
\includegraphics[width=\textwidth]{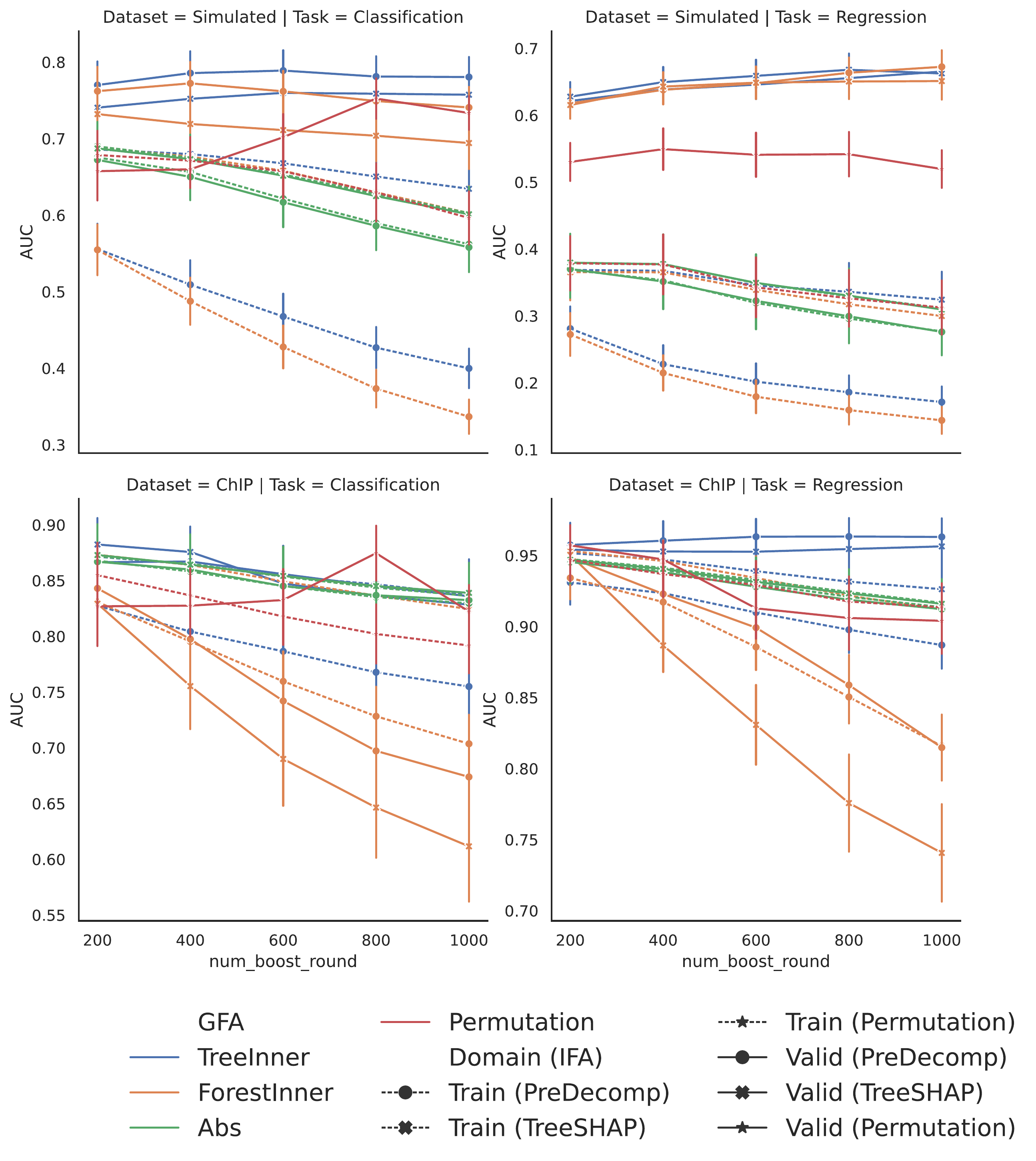}
\vspace{.3in}
\caption{
AUC score for noisy feature identification, averaged across 20 replications.
The error bars correspond to one standard error.
Higher is better.
}
\label{fig:auc-num_boost_round}
\end{figure}

\begin{figure}[ht]
\vspace{.3in}
\includegraphics[width=\textwidth]{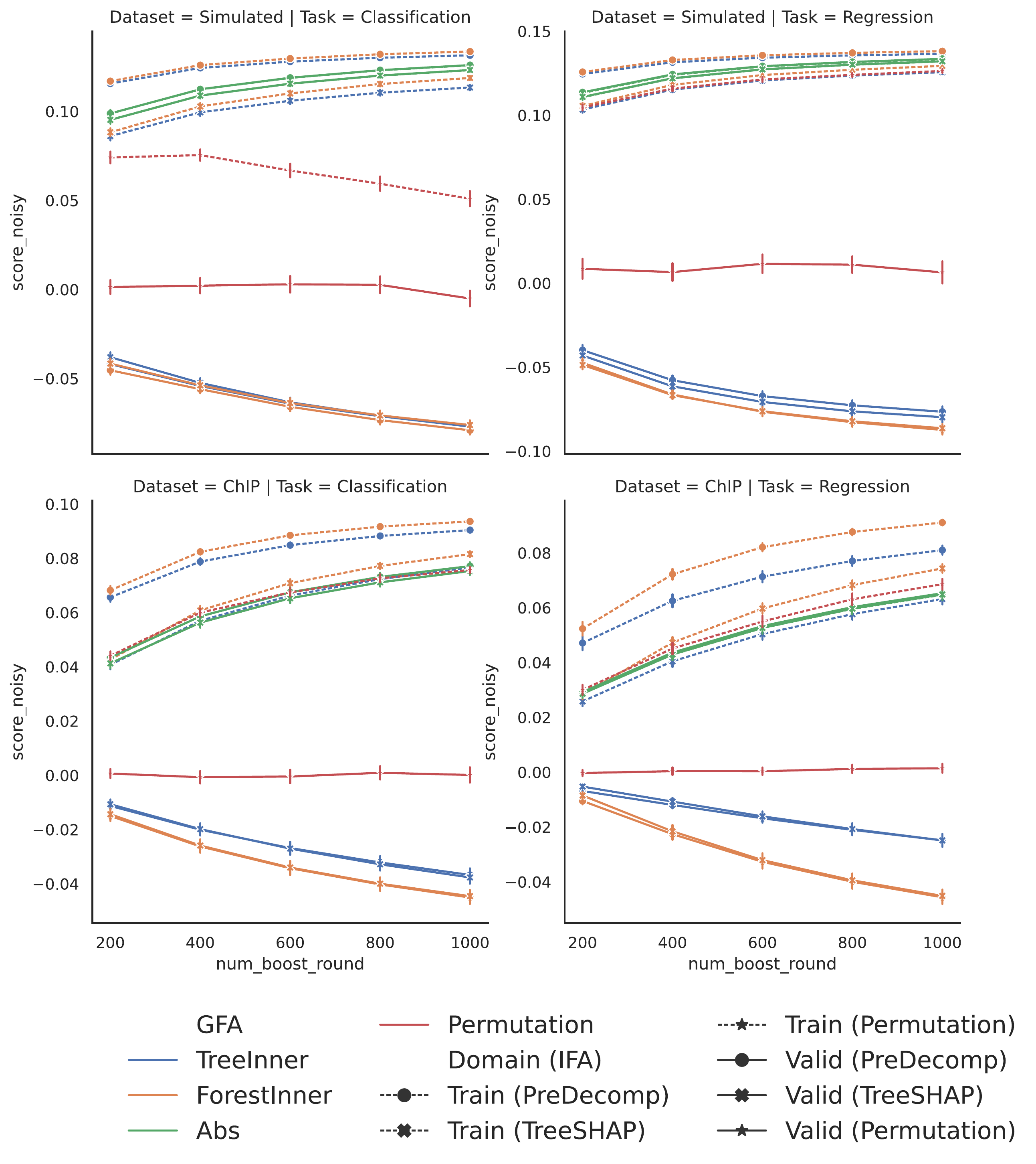}
\vspace{.3in}
\caption{
The average normalized score for the noisy features, averaged across 20 replications.
The error bars correspond to one standard error.
Lower is better.
}
\label{fig:score-noisy-num_boost_round}
\end{figure}

\begin{figure}[ht]
\vspace{.3in}
\includegraphics[width=\textwidth]{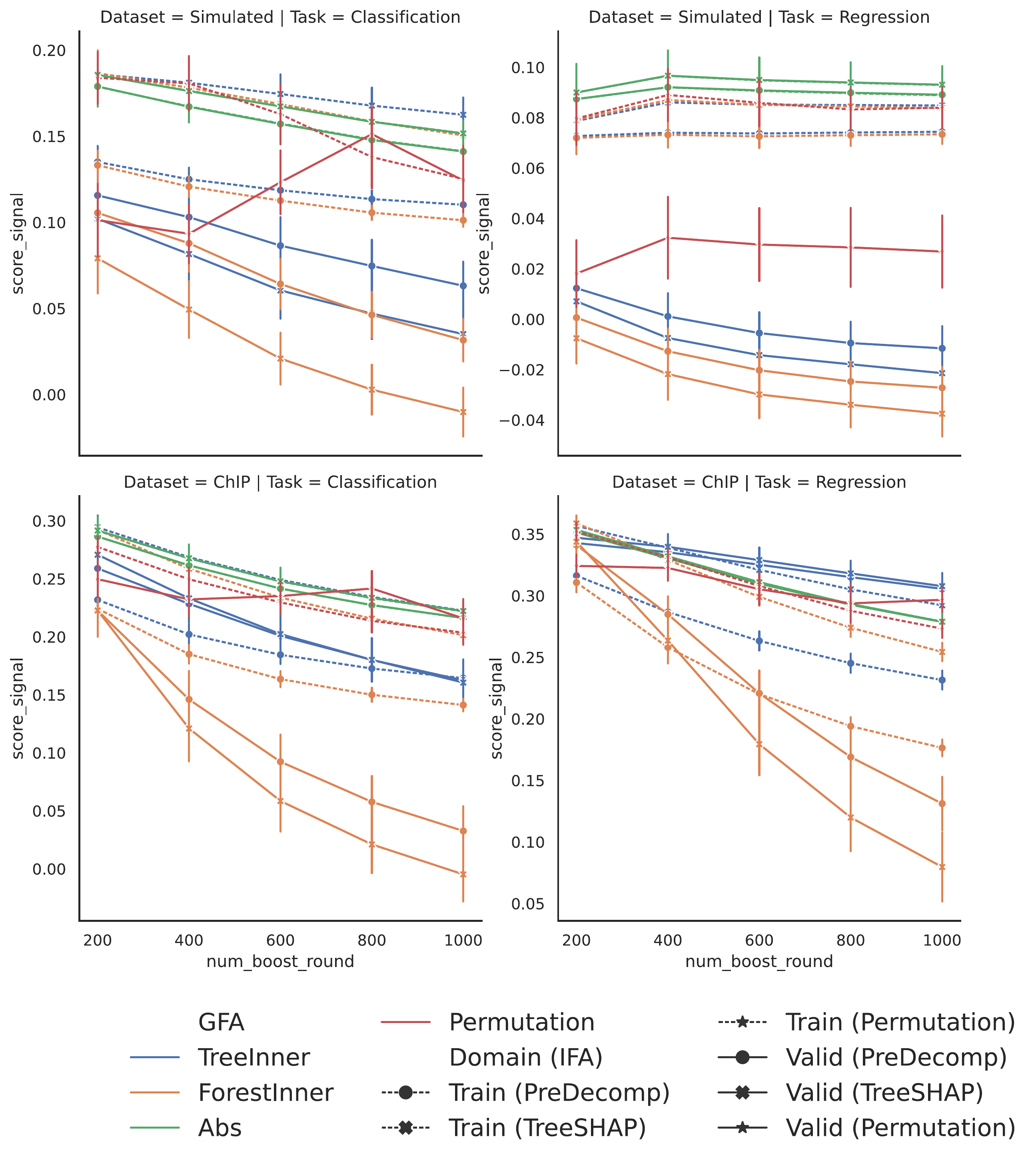}
\vspace{.3in}
\caption{
The average normalized score for the relevant features, averaged across 20 replications.
The error bars correspond to one standard error.
Lower is better.
}
\label{fig:score-signal-num_boost_round}
\end{figure}

\begin{figure}[ht]
\vspace{.3in}
\includegraphics[width=\textwidth]{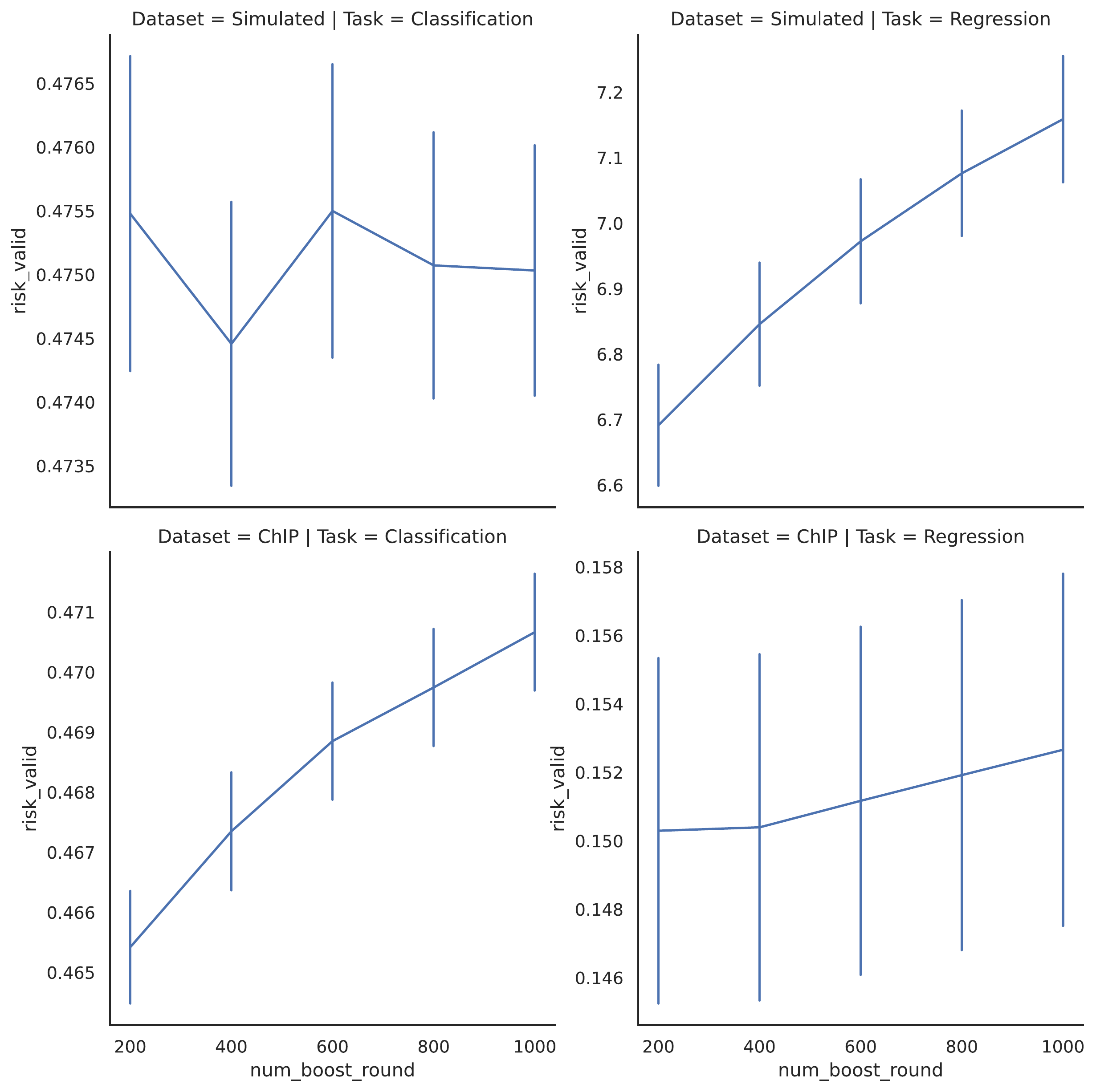}
\vspace{.3in}
\caption{
Risk evaluated on the validation set.
The error bars correspond to one standard error.
Lower is better.
}
\label{fig:risk-num_boost_round}
\end{figure}

\subsection{Sweeping over \texttt{reg\_lambda}}

\begin{figure}[ht]
\vspace{.3in}
\includegraphics[width=\textwidth]{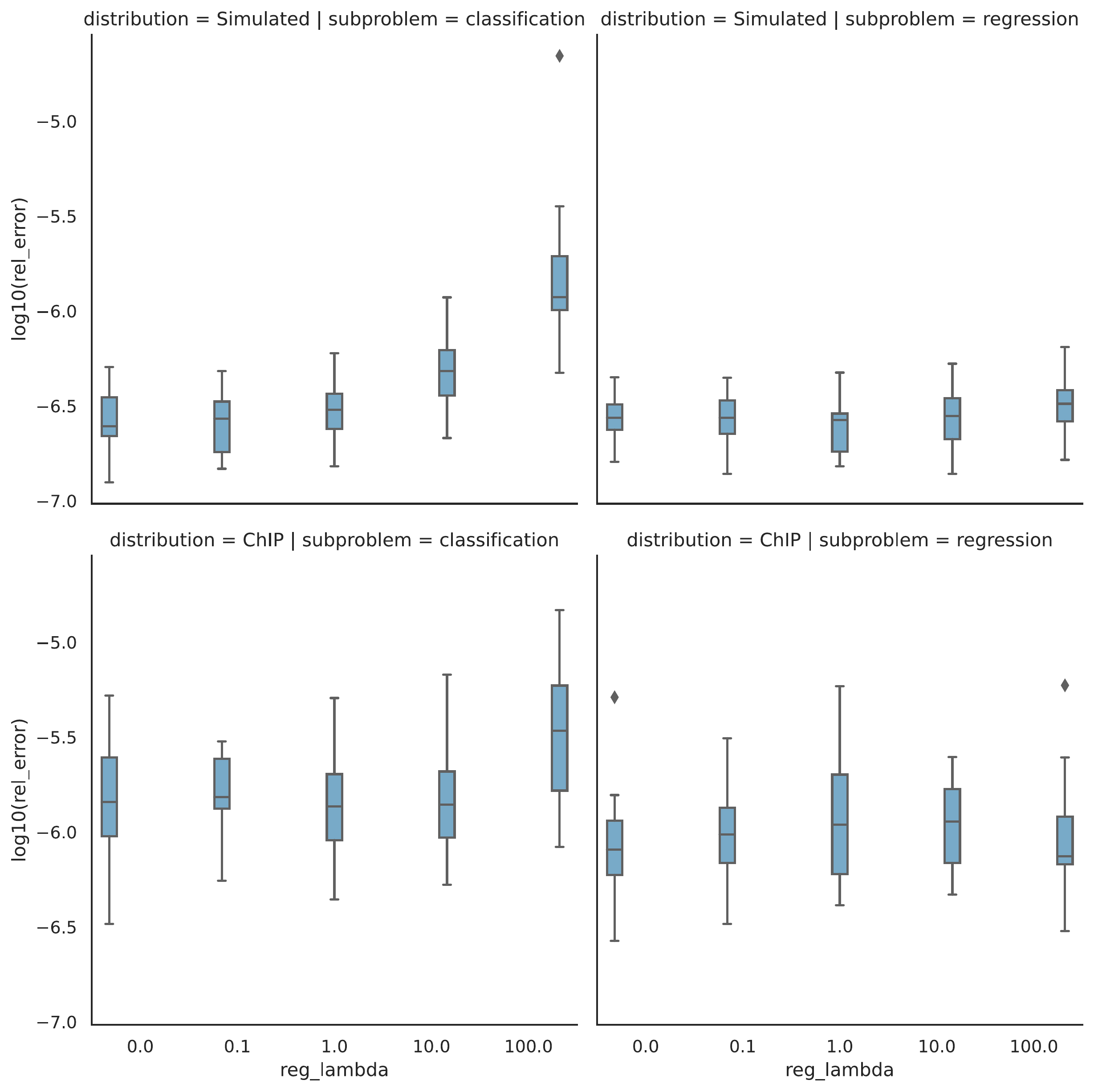}
\vspace{.3in}
\caption{Logarithm of the maximum absolute difference between the normalized total gain calculated with our methodology and the built-in method.}
\label{fig:error-reg_lambda}
\end{figure}

\begin{figure}[ht]
\vspace{.3in}
\includegraphics[width=\textwidth]{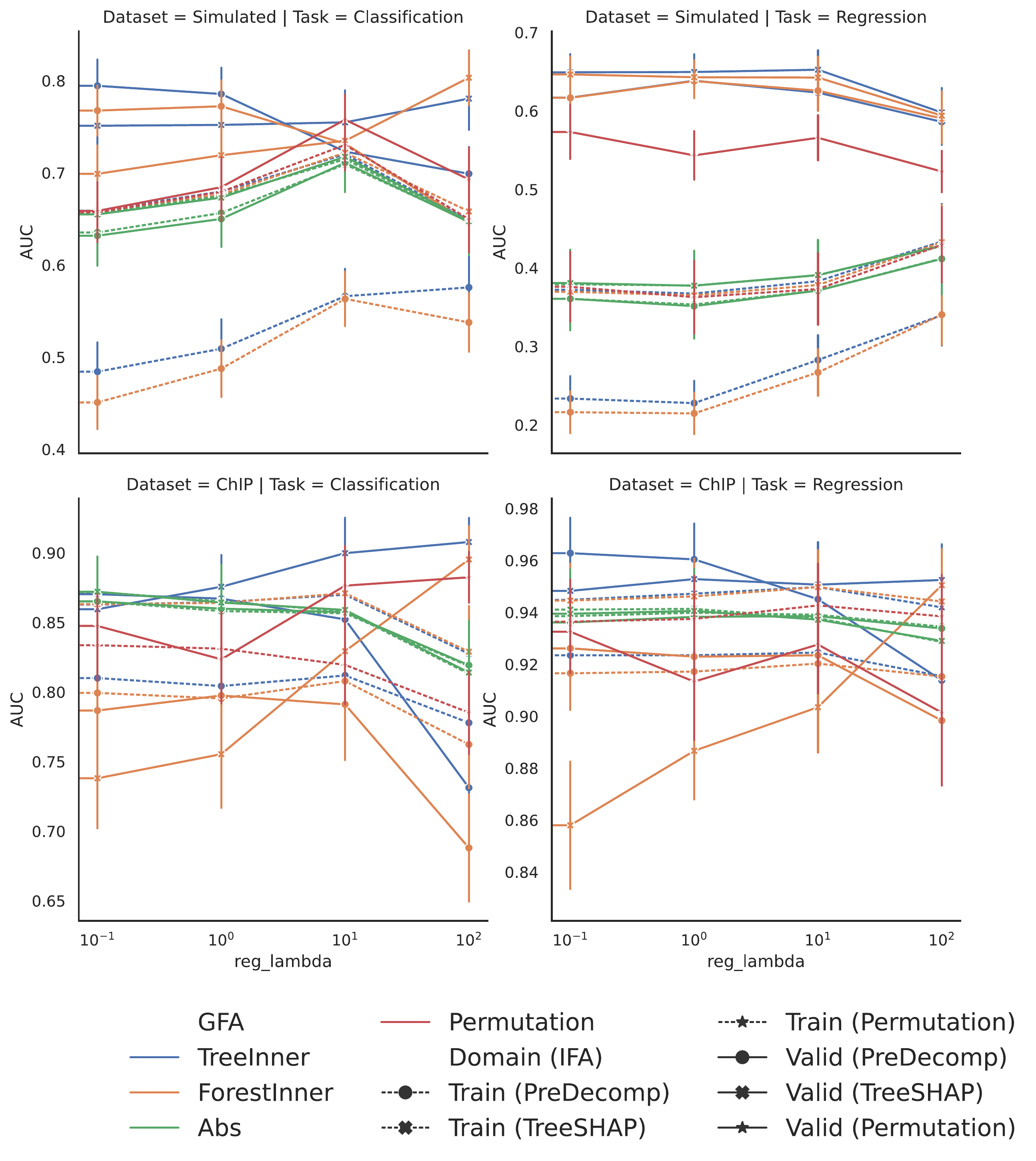}
\vspace{.3in}
\caption{
AUC score for noisy feature identification, averaged across 20 replications.
The error bars correspond to one standard error.
Higher is better.
}
\label{fig:auc-reg_lambda}
\end{figure}

\begin{figure}[ht]
\vspace{.3in}
\includegraphics[width=\textwidth]{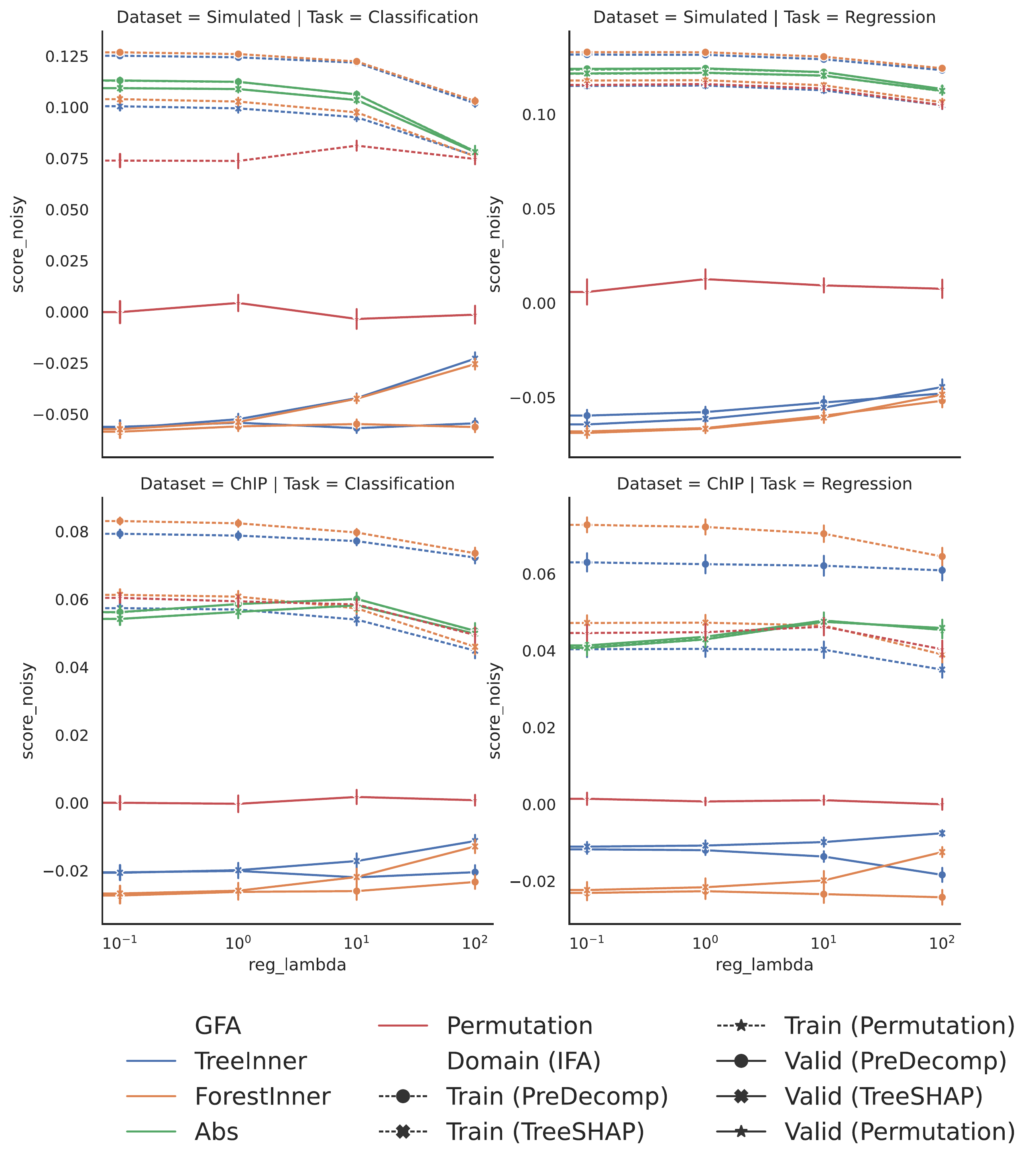}
\vspace{.3in}
\caption{
The average normalized score for the noisy features, averaged across 20 replications.
The error bars correspond to one standard error.
Lower is better.
}
\label{fig:score-noisy-reg_lambda}
\end{figure}

\begin{figure}[ht]
\vspace{.3in}
\includegraphics[width=\textwidth]{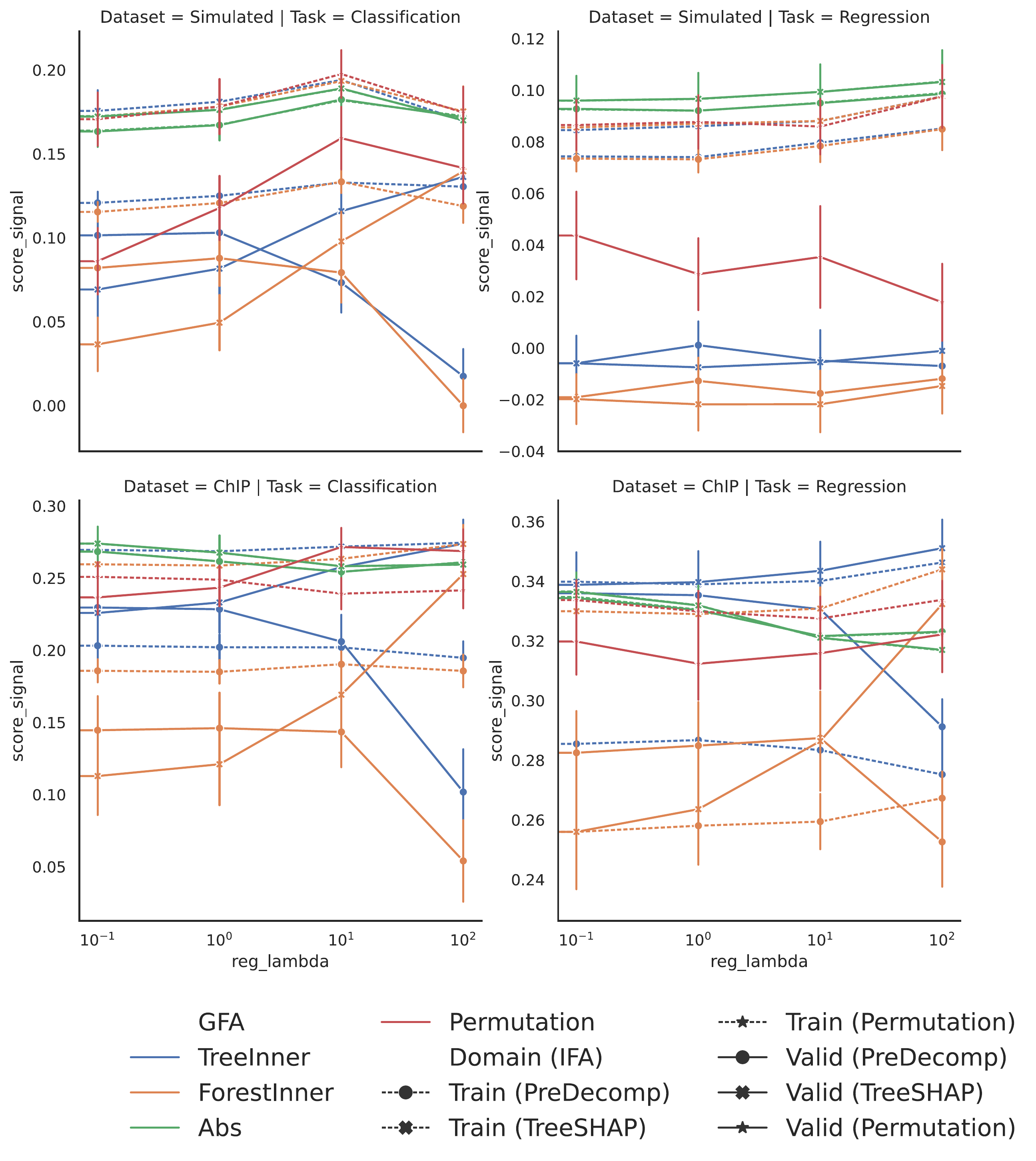}
\vspace{.3in}
\caption{
The average normalized score for the relevant features, averaged across 20 replications.
The error bars correspond to one standard error.
Lower is better.
}
\label{fig:score-signal-reg_lambda}
\end{figure}

\begin{figure}[ht]
\vspace{.3in}
\includegraphics[width=\textwidth]{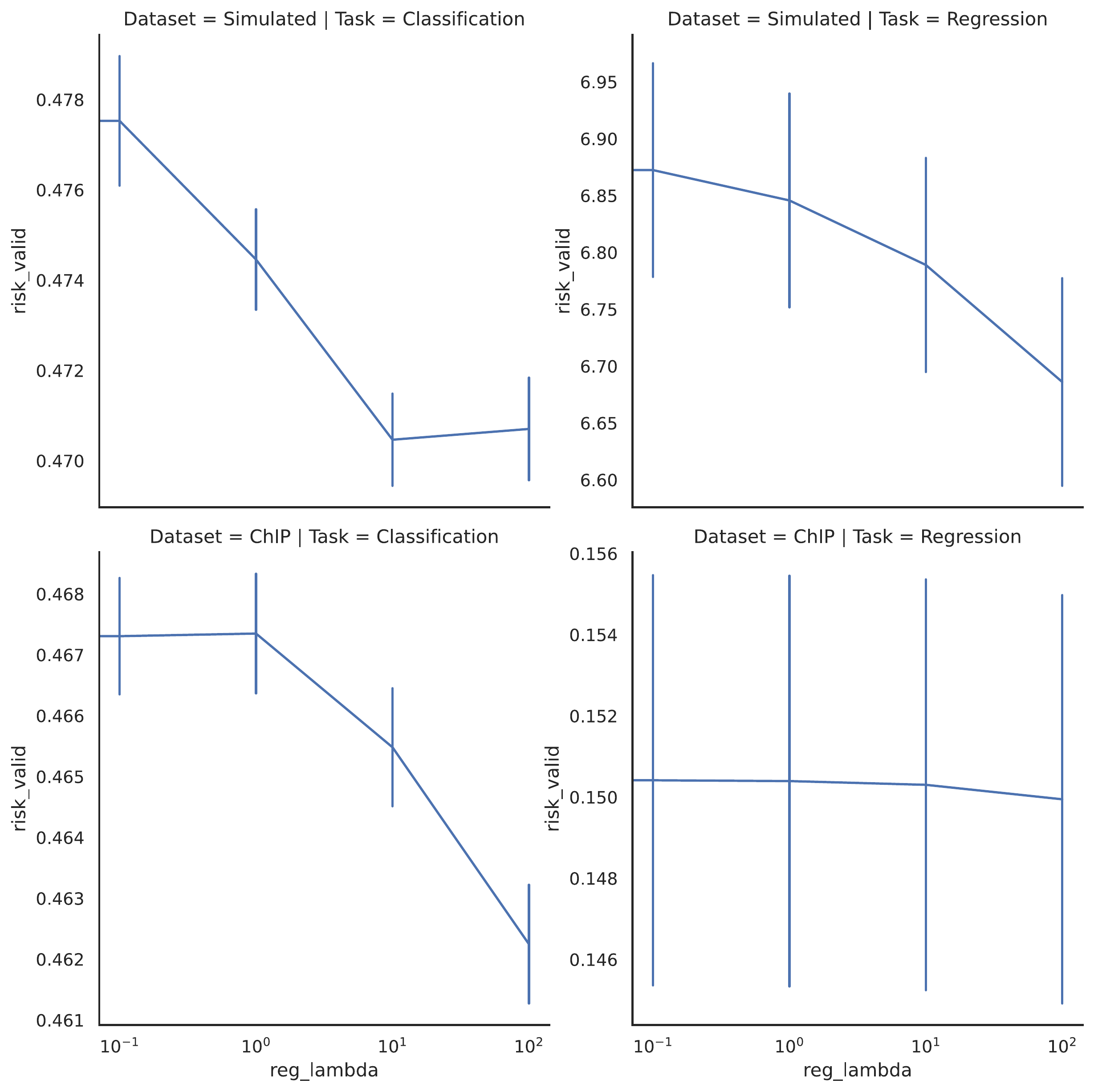}
\vspace{.3in}
\caption{
Risk evaluated on the validation set.
The error bars correspond to one standard error.
Lower is better.
}
\label{fig:risk-reg_lambda}
\end{figure}

\vfill

\end{document}